\documentclass[11pt, a4paper, logo, copyright, nonumbering]{wechat}
\PassOptionsToPackage{table,xcdraw}{xcolor}
\usepackage[numbers, square, sort&compress]{natbib}
\usepackage{dblfloatfix}
\usepackage[normalem]{ulem}
\usepackage{caption}
\definecolor{citecolor}{HTML}{1976D2}
\hypersetup{
    colorlinks=true,
    linkcolor=black,
    filecolor=magenta,
    urlcolor=blue!50!black,
    citecolor=citecolor,
}
\usepackage{dramatist}
\usepackage{xspace}
\usepackage{pifont}
\usepackage{multirow}
\usepackage{tcolorbox}
\usepackage{xltabular}
\usepackage{longtable}
\usepackage{hyperref}
\usepackage{diagbox}
\usepackage{makecell}
\usepackage{amsmath}   
\usepackage{amssymb}   
\usepackage{amsfonts}  

\usepackage{bm}
\usepackage{lineno}

\usepackage[bottom]{footmisc}

\usepackage{CJKutf8}
\usepackage{subfigure}
\usepackage{setspace}
\usepackage{subcaption} 

\usepackage{titlesec}

\usepackage{fontawesome5}
\usepackage{enumitem}

\usepackage{booktabs}
\usepackage{xcolor} 
\usepackage{graphicx}
\usepackage{mdframed} 

\usepackage{algorithm}
\usepackage{algorithmic}
\usepackage{mathtools}
\usepackage{amsthm}
\usepackage{url}
\usepackage{color}

\definecolor{color1}{HTML}{406058}
\definecolor{color2}{HTML}{78A0C8}
\definecolor{color3}{HTML}{3850A0}
\definecolor{color4}{HTML}{6860A8}
\definecolor{color5}{HTML}{A03850}
\definecolor{color6}{HTML}{AE7694}
\definecolor{color7}{HTML}{8488B4}
\hypersetup{
    colorlinks=true,
    linkcolor=black,
    citecolor=color1,
    urlcolor=color2,
    pdfborder={0 0 0}
}
\usepackage[T1]{fontenc}

\usepackage[capitalize,noabbrev]{cleveref}

\tcbset{
    contributionbox/.style={
        colback=blue!3,
        colframe=blue!50!black,
        boxrule=0.5pt,
        arc=3pt,
        left=6pt,
        right=6pt,
        top=6pt,
        bottom=6pt,
        fonttitle=\bfseries
    }
}

\titlespacing*{\paragraph}
  {0pt}                   
  {0.5ex plus 1ex minus .2ex} 
  {1em}            

\makeatletter
\def\@BTrule[#1]{%
  \ifx\longtable\undefined
    \let\@BTswitch\@BTnormal
  \else\ifx\hline\LT@hline
    \nobreak
    \let\@BTswitch\@BLTrule
  \else
     \let\@BTswitch\@BTnormal
  \fi\fi
  \global\@thisrulewidth=#1\relax
  \ifnum\@thisruleclass=\tw@\vskip\@aboverulesep\else
  \ifnum\@lastruleclass=\z@\vskip\@aboverulesep\else
  \ifnum\@lastruleclass=\@ne\vskip\doublerulesep\fi\fi\fi
  \@BTswitch}
\makeatother

\addto\extrasenglish{
}

 {\begin{list}{}%
         {\setlength{\leftmargin}{#1}}%
         \item[]%
 }
 {\end{list}}

\reportnumber{001} 

\title{\centering Gradients Must Earn Their Influence: Unifying SFT with Generalized Entropic Objectives}

\author{
    \vspace{0.5em}
    {\large \bfseries
 Zecheng Wang$^{1,2 \ast \S }$, Deyuan Liu$^{1\ast}$, Chunshan Li$^{1\dagger}$, Yupeng Zhang$^{2}$, 
    } \\ \vspace{0.6pt}
    {\large \bfseries
      Zhengyun Zhao$^{2,3}$, Dianhui Chu$^{1}$, Bingning Wang$^{2\dagger\ddagger}$, Dianbo Sui $^{1\dagger}$
    } \\ \vspace{1.0em}
    
    {\normalsize \normalfont
    $^{1}$Harbin Institute of Technology \quad
    $^{2}$WeChat, Tencent \quad
    $^{3}$Tsinghua University
    }
}
\usepackage{amsmath,amssymb,amsthm}

\providecommand{\abs}[1]{\left\lvert#1\right\rvert}

\providecommand{\R}{\mathbb{R}} %

\providecommand{\E}{{\mathbb E}}
\providecommand{\E}[1]{{\mathbb E}\left.#1\right. }        %
\providecommand{\dm}{\mathrm{d}}

\providecommand{\tt}{\mathbf{t}}

\providecommand{\mtheta}{\boldsymbol{\theta}}

\providecommand{\cC}{\mathcal{C}}

\providecommand{\cL}{\mathcal{L}}

\providecommand{\cR}{\mathcal{R}}

\providecommand{\cV}{\mathcal{V}}

\newenvironment{talign*}
{\csname align*\endcsname}
{\endalign}

\usepackage[utf8]{inputenc}         %
\usepackage[T1]{fontenc}            %
\usepackage{url}                    %
\usepackage{booktabs}               %
\usepackage{amsfonts}               %
\usepackage{nicefrac}               %
\usepackage{microtype}              %
\usepackage{xcolor}                 %
\usepackage{algorithm}
\usepackage{algorithmic}
\usepackage{graphicx}
\usepackage{subcaption}
\usepackage[flushleft]{threeparttable}
\usepackage{float}
\usepackage{multirow}
\usepackage{xspace}
\usepackage{natbib}
\usepackage{enumitem}
\usepackage[font=small]{caption}
\usepackage{autobreak}
\usepackage{sidecap}
\usepackage{wrapfig}
\usepackage{bbding}
\usepackage[toc, page, header]{appendix}
\usepackage{tikz}
\usepackage{xcolor}
\usepackage{pifont}
\usepackage{mdframed}
\usepackage{colortbl}
\usepackage{mathrsfs}
\usepackage{footmisc}

\definecolor{coral}{RGB}{255,127,80}
\definecolor{darkgreen}{RGB}{0,100,0}
\definecolor{darkyellow}{RGB}{204,153,0}
\definecolor{salmon}{RGB}{250,128,114}
\definecolor{champagne}{RGB}{247,231,206}
\definecolor{kleinblue}{RGB}{0,47,167}
\definecolor{bauhiniapurple}{RGB}{193,84,193}
\definecolor{darkred}{RGB}{150,0,0}
\definecolor{citegreen}{RGB}{0,100,50}

\newcommand{\darkredtext}[1]{{\color{darkred}#1}}

\newcommand{\thmref}[1]{\hyperref[#1]{\darkredtext{Thm.~\ref*{#1}}}}
\newcommand{\defref}[1]{\hyperref[#1]{\darkredtext{Def.~\ref*{#1}}}}
\newcommand{\propref}[1]{\hyperref[#1]{\darkredtext{Prop.~\ref*{#1}}}}
\newcommand{\assumpref}[1]{\hyperref[#1]{\darkredtext{Assump.~\ref*{#1}}}}
\newcommand{\remarkref}[1]{\hyperref[#1]{\darkredtext{Rem.~\ref*{#1}}}}
\newcommand{\hypref}[1]{\hyperref[#1]{\darkredtext{Hyp.~\ref*{#1}}}}
\newcommand{\conjref}[1]{\hyperref[#1]{\darkredtext{Conj.~\ref*{#1}}}}
\newcommand{\lemref}[1]{\hyperref[#1]{\darkredtext{Lem.~\ref*{#1}}}}
\newcommand{\corref}[1]{\hyperref[#1]{\darkredtext{Cor.~\ref*{#1}}}}
\newcommand{\noteref}[1]{\hyperref[#1]{\darkredtext{Nota.~\ref*{#1}}}}
\newcommand{\claimref}[1]{\hyperref[#1]{\darkredtext{Clm.~\ref*{#1}}}}
\newcommand{\obsref}[1]{\hyperref[#1]{\darkredtext{Obs.~\ref*{#1}}}}

\newcommand{\transparentteal}[1]{%
  \tikz[baseline=(X.base)] \node[fill=teal, fill opacity=0.1, text opacity=1, inner sep=2pt, outer sep=0pt] (X) {#1};%
}
\newcommand{\figref}[1]{\hyperref[#1]{\transparentteal{Figure~\ref*{#1}}}}

\newcommand{\transparentdarkgreen}[1]{%
  \tikz[baseline=(X.base)] \node[fill=darkgreen, fill opacity=0.1, text opacity=1, inner sep=2pt, outer sep=0pt] (X) {#1};%
}
\newcommand{\tabref}[1]{\hyperref[#1]{\transparentdarkgreen{Table~\ref*{#1}}}}

\newcommand{\transparentdarkyellow}[1]{%
  \tikz[baseline=(X.base)] \node[fill=darkyellow, fill opacity=0.1, text opacity=1, inner sep=2pt, outer sep=0pt] (X) {#1};%
}
\newcommand{\secref}[1]{\hyperref[#1]{\transparentdarkyellow{Section~\ref*{#1}}}}

\newcommand{\transparentcoral}[1]{%
  \tikz[baseline=(X.base)] \node[fill=coral, fill opacity=0.1, text opacity=1, inner sep=2pt, outer sep=0pt] (X) {#1};%
}
\newcommand{\appref}[1]{\hyperref[#1]{\transparentcoral{Appendix~\ref*{#1}}}}

\newcommand{\transparentkleinblue}[1]{%
  \tikz[baseline=(X.base)] \node[fill=kleinblue, fill opacity=0.1, text opacity=1, inner sep=2pt, outer sep=0pt] (X) {#1};%
}
\newcommand{\algref}[1]{\hyperref[#1]{\transparentkleinblue{Algorithm~\ref*{#1}}}}

\newcommand{\transparentbauhiniapurple}[1]{%
  \tikz[baseline=(X.base)] \node[fill=bauhiniapurple, fill opacity=0.1, text opacity=1, inner sep=2pt, outer sep=0pt] (X) {#1};%
}
\newcommand{\equref}[1]{\hyperref[#1]{\transparentbauhiniapurple{Equation~\ref*{#1}}}}

\newtheoremstyle{custom}
{1pt} 
{1pt} 
{\itshape} 
{} 
{\bfseries} 
{} 
{ } 
{\thmname{#1} \thmnumber{#2} \thmnote{(#3)} . } 

\theoremstyle{custom}

\newtheorem{innerdefinition}{Definition}
\newtheorem{innerproposition}{Proposition}
\newtheorem{innerassumption}{Assumption}
\newtheorem{innerremark}{Remark}
\newtheorem{innertheorem}{Theorem}
\newtheorem{innerhypothesis}{Hypothesis}
\newtheorem{innerconjecture}{Conjecture}
\newtheorem{innerlemma}{Lemma}
\newtheorem{innercorollary}{Corollary}
\newtheorem{innerexample}{Example}
\newtheorem{innernotation}{Notation}
\newtheorem{innerclaim}{Claim}
\newtheorem{innerproblem}{Problem}

\newtheorem{innerobservation}{Observation}

\newmdenv[
  backgroundcolor=gray!10,
  linecolor=gray!100,
  linewidth=0.01pt,
  skipabove=2pt,
  skipbelow=2pt,
  innertopmargin=10pt,
  innerbottommargin=5pt,
  innerleftmargin=5pt,
  innerrightmargin=5pt,
]{definitionframe}

\newmdenv[
  backgroundcolor=blue!10,
  linecolor=blue!100,
  linewidth=0.01pt,
  skipabove=2pt,
  skipbelow=2pt,
  innertopmargin=10pt,
  innerbottommargin=5pt,
  innerleftmargin=5pt,
  innerrightmargin=5pt,
]{propositionframe}

\newmdenv[
  backgroundcolor=green!10,
  linecolor=green!100,
  linewidth=0.01pt,
  skipabove=2pt,
  skipbelow=2pt,
  innertopmargin=10pt,
  innerbottommargin=5pt,
  innerleftmargin=5pt,
  innerrightmargin=5pt,
]{assumptionframe}

\newmdenv[
  backgroundcolor=yellow!10,
  linecolor=yellow!100,
  linewidth=0.01pt,
  skipabove=2pt,
  skipbelow=2pt,
  innertopmargin=10pt,
  innerbottommargin=5pt,
  innerleftmargin=5pt,
  innerrightmargin=5pt,
]{remarkframe}

\newmdenv[
  backgroundcolor=red!10,
  linecolor=red!100,
  linewidth=0.01pt,
  skipabove=2pt,
  skipbelow=2pt,
  innertopmargin=10pt,
  innerbottommargin=5pt,
  innerleftmargin=5pt,
  innerrightmargin=5pt,
]{theoremframe}

\newmdenv[
  backgroundcolor=purple!10,
  linecolor=purple!100,
  linewidth=0.01pt,
  skipabove=2pt,
  skipbelow=2pt,
  innertopmargin=10pt,
  innerbottommargin=5pt,
  innerleftmargin=5pt,
  innerrightmargin=5pt,
]{hypothesisframe}

\newmdenv[
  backgroundcolor=orange!10,
  linecolor=orange!100,
  linewidth=0.01pt,
  skipabove=2pt,
  skipbelow=2pt,
  innertopmargin=10pt,
  innerbottommargin=5pt,
  innerleftmargin=5pt,
  innerrightmargin=5pt,
]{conjectureframe}

\newmdenv[
  backgroundcolor=cyan!10,
  linecolor=cyan!100,
  linewidth=0.01pt,
  skipabove=2pt,
  skipbelow=2pt,
  innertopmargin=10pt,
  innerbottommargin=5pt,
  innerleftmargin=5pt,
  innerrightmargin=5pt,
]{lemmaframe}

\newmdenv[
  backgroundcolor=magenta!10,
  linecolor=magenta!100,
  linewidth=0.01pt,
  skipabove=2pt,
  skipbelow=2pt,
  innertopmargin=10pt,
  innerbottommargin=5pt,
  innerleftmargin=5pt,
  innerrightmargin=5pt,
]{corollaryframe}

\newmdenv[
  backgroundcolor=lime!10,
  linecolor=lime!100,
  linewidth=0.01pt,
  skipabove=2pt,
  skipbelow=2pt,
  innertopmargin=10pt,
  innerbottommargin=5pt,
  innerleftmargin=5pt,
  innerrightmargin=5pt,
]{exampleframe}

\newmdenv[
  backgroundcolor=pink!10,
  linecolor=pink!100,
  linewidth=0.01pt,
  skipabove=2pt,
  skipbelow=2pt,
  innertopmargin=10pt,
  innerbottommargin=5pt,
  innerleftmargin=5pt,
  innerrightmargin=5pt,
]{notationframe}

\newmdenv[
  backgroundcolor=violet!10,
  linecolor=violet!100,
  linewidth=0.01pt,
  skipabove=2pt,
  skipbelow=2pt,
  innertopmargin=10pt,
  innerbottommargin=5pt,
  innerleftmargin=5pt,
  innerrightmargin=5pt,
]{claimframe}

\newmdenv[
  backgroundcolor=salmon!10,
  linecolor=salmon!100,
  linewidth=0.01pt,
  skipabove=2pt,
  skipbelow=2pt,
  innertopmargin=10pt,
  innerbottommargin=5pt,
  innerleftmargin=5pt,
  innerrightmargin=5pt,
]{problemframe}

\newmdenv[
  backgroundcolor=lavender!10,
  linecolor=lavender!100,
  linewidth=0.01pt,
  skipabove=2pt,
  skipbelow=2pt,
  innertopmargin=10pt,
  innerbottommargin=5pt,
  innerleftmargin=5pt,
  innerrightmargin=5pt,
]{observationframe}

\newenvironment{definition}
{\begin{definitionframe}\begin{innerdefinition}}
      {\end{innerdefinition}\end{definitionframe}}

\newenvironment{proposition}
{\begin{propositionframe}\begin{innerproposition}}
      {\end{innerproposition}\end{propositionframe}}

\newenvironment{assumption}
{\begin{assumptionframe}\begin{innerassumption}}
      {\end{innerassumption}\end{assumptionframe}}

\newenvironment{remark}
{\begin{remarkframe}\begin{innerremark}}
      {\end{innerremark}\end{remarkframe}}

\newenvironment{theorem}
{\begin{theoremframe}\begin{innertheorem}}
      {\end{innertheorem}\end{theoremframe}}

\newenvironment{lemma}
{\begin{lemmaframe}\begin{innerlemma}}
      {\end{innerlemma}\end{lemmaframe}}

\newenvironment{corollary}
{\begin{corollaryframe}\begin{innercorollary}}
      {\end{innercorollary}\end{corollaryframe}}

\usepackage[textwidth=3.5cm]{todonotes}

\usepackage{multirow}
\usepackage{makecell}
\usepackage{colortbl}

\usepackage{bm}

\definecolor{ForestGreen}{RGB}{34,139,34}
\definecolor{OrangeRed}{RGB}{255,69,0}

\definecolor{maroon}{RGB}{128,0,0}

\definecolor{FailureRed}{RGB}{192,0,0}

\usepackage{bm}

\definecolor{ForestGreen}{RGB}{34,139,34}
\definecolor{OrangeRed}{RGB}{255,69,0}

\DeclareMathOperator{\softmax}{softmax}
\definecolor{goodgreen}{HTML}{38761D}
\definecolor{badred}{HTML}{CC0000}   
\definecolor{strongbadred}{HTML}{A50000} 
\definecolor{highlightblue}{HTML}{0072B2}
\definecolor{lightrow}{gray}{0.95}    
\definecolor{white}{rgb}{1,1,1} 

\definecolor{iclBaselineBg}{HTML}{FFF2CC} 
\definecolor{iclHyperBg}{HTML}{D4EBEB}    
\definecolor{iclNoneBg}{HTML}{EAEAEA}      
\usepackage{siunitx}     

\usepackage{mathtools}

\github{https://github.com/luludus/DEFT}

\newcommand\blfootnote[1]{%
  \begingroup
  \renewcommand\thefootnote{}\footnote{#1}%
  \addtocounter{footnote}{-1}%
  \endgroup
}

\begin{abstract}
Standard negative log-likelihood (NLL) for Supervised Fine-Tuning (SFT) applies uniform token-level weighting. This rigidity creates a two-fold failure mode: (i) overemphasizing low-probability targets can amplify gradients on noisy supervision and disrupt robust priors, and (ii) uniform weighting provides weak sharpening when the model is already confident.
Existing methods fail to resolve the resulting plasticity--stability dilemma, often suppressing necessary learning signals alongside harmful ones.
To address this issue, we unify token-level SFT objectives within a generalized deformed-log family and expose a universal \emph{gate $\times$ error} gradient structure, where the gate controls how much the model trusts its current prediction.
By employing the Cayley transform, we map the model's continuously evolving uncertainty onto a continuous focus trajectory, which enables seamless interpolation between scenarios involving uncertain novel concepts and those involving well-established knowledge.
We then introduce \textbf{Dynamic Entropy Fine-Tuning (DEFT)}, a parameter-free objective that modulates the trust gate using distribution concentration (Rényi-2 entropy) as a practical proxy for the model's predictive state. Extensive experiments and analyses demonstrate that DEFT achieves a better balance between exploration and exploitation, leading to improved overall performance.
\end{abstract}

\begin{document}

\begin{CJK*}{UTF8}{gbsn}

\maketitle

\enlargethispage{1cm}

\blfootnote{$^\dagger$ Corresponding Author.}
\blfootnote{$^\ddagger$ Project Leader.}
\blfootnote{$^\ast$ Equal contribution.}
\blfootnote{$\S$ Work done during internship at WeChat.}

\section{Introduction}
\label{sec:intro}

In the contemporary artificial intelligence paradigm, Supervised Fine-Tuning (SFT) serves as a pivotal bridge connecting the general capabilities of pre-trained Large Language Models (LLMs) with domain-specific applications.  
While inheriting the next-token prediction paradigm effectively aligns models with instruction-following formats, a growing body of research indicates that the standard negative log-likelihood (NLL) objective is suboptimal for the fine-tuning stage~\cite{diao2026entropyadaptivefinetuningresolvingconfident,chen2025stepwiseadaptiveintegrationsupervised,li2025loglikelihoodprobabilitybasedobjectives}.

The fundamental theoretical limitation of the NLL objective lies in its indiscriminate treatment of tokens. 
NLL implicitly assumes uniform token importance, allocating maximal gradients to low-probability targets regardless of whether they represent learnable knowledge gaps or spurious supervision. 
Such a unified gating mechanism can amplify the inevitable noise in the data~\cite{taheri2025forgettingnewmechanismbetter,dodge2021documentinglargewebtextcorpora}, thereby impairing the model's out-of-distribution generalization ability.
Critically, SFT data often contains \emph{confident conflicts}: cases where a pretrained model makes a high-confidence, low-entropy prediction that disagrees with the target. 
Forcing the model to fit such low-probability targets may induce large parameter updates and potentially lead to catastrophic forgetting~\cite{diao2026entropyadaptivefinetuningresolvingconfident,wu2025mitigatingforgettingllmfinetuning}.
Besides, as the model grows confident in many tokens, NLL gradients decay linearly with the error, leading to inefficient distribution sharpening and diminishing returns.

To improve SFT robustness, one line of work performs fine-grained data selection or token masking to separate informative signals from noise~\citep{pang2025tokencleaningfinegraineddata,liu2026profitleveraginghighvaluesignals,ruan2025enhancinglargelanguagemodel}. 
These methods can increase the signal-to-noise ratio but often rely on costly offline computation or static heuristics that cannot adapt to the model's evolving state during training. 
A complementary line modifies the objective itself to reweight token-level gradients. 
\begin{figure*}[t]
\centering
\includegraphics[width=.95\textwidth]{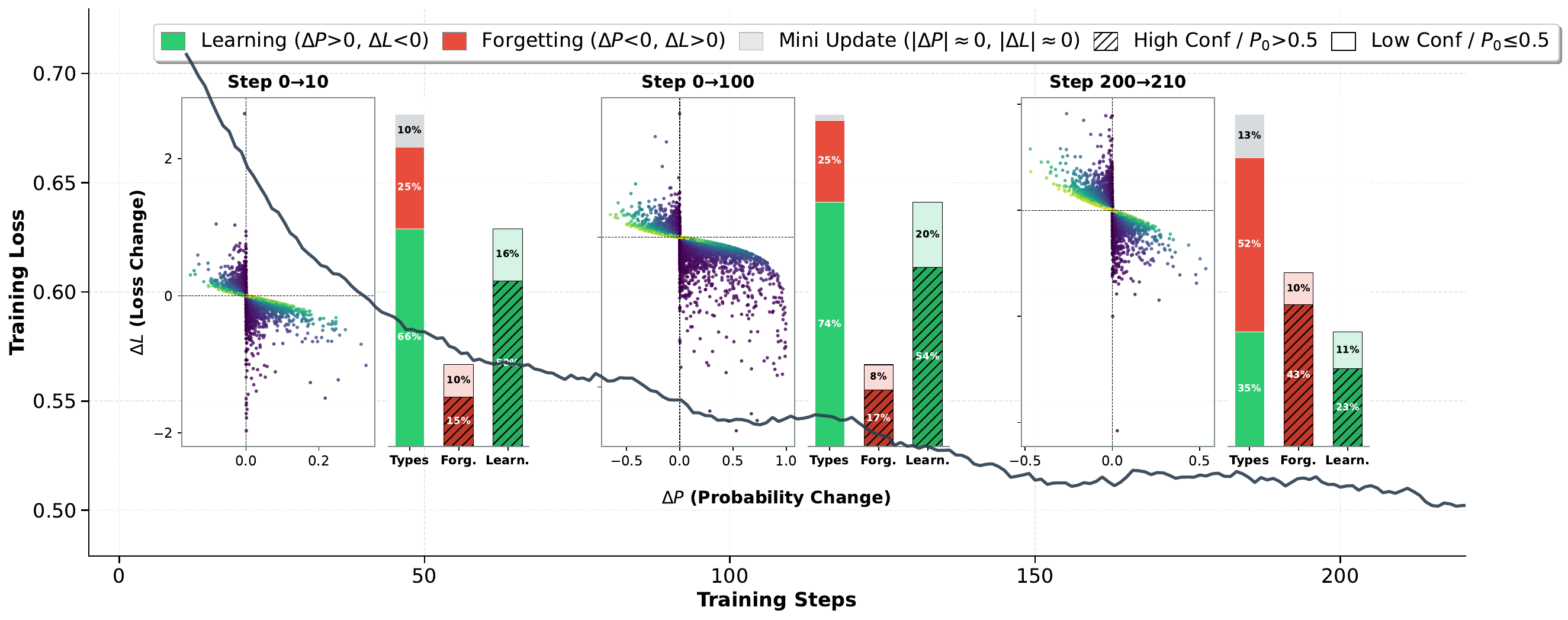} 
\caption{Learning conflicts of different tokens under NLL during training (token-wise learning and forgetting). The scatter plot shows token probability ($\triangle P$) and loss ($\triangle L$) changes. Bar plots show the proportions of tokens, including overall learning and forgetting, the fractions of forgotten tokens with high and low confidence (Q2 in the scatter plot), and fractions of learning tokens with high and low confidence (Q4 in the scatter plot). Hatched patterns indicate high-confidence tokens.}
\label{fig:overview}
\end{figure*}
Confidence-based linear scaling~\citep{li2025beyond, wu2025generalizationsftreinforcementlearning} reduces the update weights of low-probability tokens to sharpen high-probability token distributions, thereby improving model performance in high-competence regions such as mathematical reasoning. Entropy adaptive fine-tuning~\citep{diao2026entropyadaptivefinetuningresolvingconfident} alleviate conflicts by mitigating the impact of large gradients from low-probability, low-entropy tokens. However, the former may suppress the learning of new task knowledge in the early stages of fine-tuning, while the latter fails to address inefficient sharpening under high-confidence predictions. Consequently, how to balance exploration and exploitation during SFT optimization remains an open challenge.

\begin{mdframed}[
    hidealllines=true,
    backgroundcolor=gray!8,
    innerleftmargin=14pt,
    innerrightmargin=14pt,
    innertopmargin=6pt,
    innerbottommargin=8pt,
    roundcorner=2pt,
    linewidth=0.6pt,
    linecolor=gray!60
]
\noindent
\textbf{$\mathbb{Q}$: How can we design an SFT objective that preserves coverage for uncertain knowledge gaps while efficiently sharpening confident predictions and suppressing confident conflicts, \emph{without} expensive offline token filtering?}
\end{mdframed}

We address this question by unifying token-level SFT objectives within a deformed-log family based on the Tsallis $q$-logarithm~\citep{tsallis1988possible}, which exposes a universal \emph{gate $\times$ error} gradient structure. 
This view reveals that the gate term controls how much the model trusts its current prediction, and that no single static gate can simultaneously achieve coverage at low confidence and sharpening at high confidence in post-training. 
To resolve this, we derive a state-dependent focusing trajectory using geometric principles and the Cayley transform, yielding a dynamic trust gate that smoothly transitions from NLL-like behavior when the model is uncertain to probability-loss-like behavior when the model is confident. 
We further show this trajectory can be interpreted as normalized error surprisal gating through information-theoretic analysis. 
Finally, we propose \textbf{Dynamic Entropy Fine-Tuning (DEFT)}, a parameter-free objective that approximates this state-dependent gating by modulating gradients via predictive entropy. 

\noindent\textbf{Empirical validation.} To comprehensively evaluate and analyze the mechanisms of different loss functions, following~\cite{li2025beyond}, we partition the training data into three regimes based on the model’s domain-specific priors: Model-Strong, Model-Intermediate, and Model-Weak.
Our extensive evaluation spans seven model backbones and multiple domain-specific benchmarks. The results demonstrate that both the Cayley Transform and DEFT achieve superior overall performance across all three regimes. Notably, DEFT consistently outperforms the standard NLL objective across nearly all domains, showing particularly significant margins in the Model-Strong and Model-Weak regions.

\noindent\textbf{Our main contributions are:}
\begin{enumerate}[label=\alph*., nosep, leftmargin=16pt]
  \item \textbf{Unified objective framework:} 
We present a deformed-log family for token-level SFT that makes learning-signal allocation explicit through a confidence-dependent trust gate, unifying prior objectives and revealing their implicit allocation strategies.
  \item \textbf{State-dependent trust gating:}
  We derive a parameter-free focusing trajectory via the Cayley transform, interpolating between coverage and sharpening to address the exploration–explo itation tension in SFT.
\item \textbf{DEFT objective and strong empirical results:} 
We propose DEFT, a practical method that adaptively modulates the focus index via distribution-level concentration, yielding consistent and robust gains over strong baselines across the model capability spectrum.
\end{enumerate}

\section{Related Works}

In recent years, the rapid progress of LLMs has substantially advanced general-purpose AI capabilities~\cite{grattafiori2024llama3herdmodels,qwen2025qwen25technicalreport,shao2024deepseekmathpushinglimitsmathematical,wang2025baichuanm1pushingmedicalcapability,openai2024openaio1card}. In the dominant post-training paradigm, aligning pretrained models with human preferences and domain needs—while enhancing reasoning—remains a central challenge~\cite{cai2025trainingfreegrouprelativepolicy,zhang2025instructiontuninglargelanguage,rafailov2024directpreferenceoptimizationlanguage,wang2025vpo}. SFT, as a core stage, serves as a key bridge from pretraining to downstream performance. However, the standard NLL objective in SFT applies uniform token weighting, ignoring model confidence and data quality, leading to a granularity mismatch. To improve the training signal-to-noise ratio, prior work mainly follows two lines: token-level valuation/filtering and loss reweighting based on uncertainty or predicted probability.

\paragraph{Token-level modeling and filtering.}
Recent studies advance data quality control from coarse, sample-level filtering to finer, token-level interventions by estimating the training value of each token. Specifically, \citet{pang2025tokencleaningfinegraineddata} evaluate token contributions to parameter updates or performance variations and perform separation/cleaning via thresholding. \citet{taheri2025forgettingnewmechanismbetter} classify tokens as beneficial vs.\ harmful and impose explicit forgetting/suppression on the latter. Moreover, \cite{ruan2025enhancinglargelanguagemodel} use counterfactual perturbations to identify critical reasoning tokens and selectively fine-tune them, improving stability and generalization on complex reasoning tasks. \cite{liu2026profitleveraginghighvaluesignals} further show that high-probability tokens often encode core logical structure, whereas low-probability tokens are more interchangeable; selectively masking such tokens can improve reasoning performance. Despite their effectiveness, these approaches typically require costly offline computation or heuristic preprocessing, and their static/hard masking mechanisms are limited in adapting to the model’s evolving knowledge boundaries during training.

\paragraph{Modifying SFT losses via probability/uncertainty-aware reweighting.}
Beyond explicit token filtering, another line of work revises the NLL objective to reweight token losses using model confidence signals. \citet{li2025beyond,wu2025generalizationsftreinforcementlearning} linearly scale token losses by predicted probabilities, which can improve performance in high-capability regions where the model is already confident and relevant pretraining coverage is abundant (e.g., MATH). By down-weighting low-probability tokens, these methods prevent such tokens from dominating parameter updates, but the reduced gradients can hinder adaptation when learning genuinely new tasks. \cite{diao2026entropyadaptivefinetuningresolvingconfident} incorporate entropy into NLL to mitigate overfitting on low-entropy, low-probability tokens, improving stability and generalization; however, this formulation does not directly control the effect of token probability on update magnitudes, so low-probability tokens may still dominate updates. 
In contrast, we propose Dynamic Entropy Fine-Tuning (DEFT): by using entropy to modulate a dynamic trust gate, the objective smoothly transitions to NLL-like behavior when the model is uncertain and to probability-loss-like behavior when the model is confident, thereby achieving a better exploration–exploitation balance.

\section{Learning Signal Allocation in SFT}
\label{sec:dilemma}

This study focuses on the post-pretraining stage of SFT. In this setting, the model is initialized from a pre-trained base $p_{\mtheta_0}$ that already encodes rich domain prior knowledge. Given that the supervision signals often contain noise or are suboptimal, the core challenge of SFT lies in how the objective function addresses the mismatch between token-level optimization granularity and the model’s intrinsic knowledge state.
Let a raw SFT example be an input--output pair $(x,\tilde{\mathbf y})$, where $\tilde{\mathbf y}=(\tilde y_1,\dots,\tilde y_N)$ is a target sequence.
We unroll sequences into token-level training pairs by defining the token context at position $t$ as:
\begin{equation}
c_t \;\triangleq\; (x,\tilde y_{<t}),
\end{equation}
and the corresponding supervised target token as $\tilde y_t\in\cV$.
This induces a token-level training distribution $T$ over pairs $(c,\tilde y)$.
Given a token context $c$, the model predicts a next-token distribution:
\begin{equation}
z_{\mtheta}(c)\in\mathbb{R}^{\abs{\cV}},
\quad
p_{\mtheta}(\cdot\mid c)=\mathrm{softmax}\!\bigl(z_{\mtheta}(c)\bigr)\in\Delta^{\abs{\cV}-1}.
\end{equation}
We formalize the supervised target-token probability by:
\begin{equation}
p \;\triangleq\; p_{\mtheta}(\tilde y\mid c)\in(0,1].
\end{equation}

\subsection{Token-Level Objectives and Gradient Structure}

For any differentiable, nonincreasing function $f:[0,1]\to\mathbb{R}$, we define the token-level SFT objective as:
\begin{equation}
\label{eq:general_obj}
\mathcal{L}_f(\mtheta)
\;=\;
\mathbb{E}_{(c,\tilde y)\sim T}\!\left[f\!\left(p_{\mtheta}(\tilde y\mid c)\right)\right].
\end{equation}
The standard NLL objective corresponds to $f(p)=-\log p$:
\begin{equation}
\mathcal{L}_{\mathrm{NLL}}(\mtheta)
= \E_{(c,\tilde y)\sim T}\!\left[ -\log p_{\mtheta}(\tilde y\mid c)\right].
\end{equation}
An objective determines learning dynamics through its gradients on the logits.
For a general differentiable, nonincreasing $f$, the softmax-logit gradient takes the form:
\begin{equation}
\label{eq:general_grad}
\frac{\partial \bigl(\mathcal{L}_f\bigr)}{\partial z_i(c)}
\;=\;
s_f\!\left(p\right)\,\bigl(p_{\mtheta}(i\mid c)-\delta_{i,\tilde y}\bigr)
\end{equation}
where $s_f(p)\;\triangleq\;-f'(p)\,p \;\geq 0$, so the \emph{magnitude} of the learning signal on the target logit is
\begin{equation}
\label{eq:learning_signal}
W_f(p)\;\triangleq\;\left|\frac{\partial \bigl(\mathcal{L}_f\bigr)}{\partial z_{\tilde y}(c)}\right|
\;=\;
s_f(p)\,(1-p).
\end{equation}
Eq.~\eqref{eq:learning_signal} exposes a universal \(\text{gate} \times \text{error}\) structure: \((1-p)\) measures prediction error, while \(s_f(p)\) acts as a confidence-dependent \emph{gate} that decides how much the model should trust (and learn from) the current token. Detailed proofs are provided in \appref{appen:gradient_structure}.

\subsection{The Dual Nature of Low Probability: Coverage vs. Sharpening}
\label{subsec:dual_nature}

For NLL, $f(p)=-\log p$ implies $f'(p)=-1/p$ and therefore \(s_f(p)=1\).
Hence \(W_{\mathrm{NLL}}(p)=1-p\), or equivalently the target-logit gradient is
\begin{equation}
\frac{\partial\,(-\log p)}{\partial z_{\tilde y}} = -(1-p).
\end{equation}
This induces an intrinsic \emph{uniform gating} property: the gate $s_f(p)=1$ is constant across all confidence levels, maximizing update magnitude precisely when the model is least confident, while providing negligible signal for high-confidence predictions.
However, during post-training, low target probabilities are inherently ambiguous. They may arise from two distinct sources: a \emph{learnable gap} (valid knowledge the model has not yet acquired) or an \emph{adversarial conflict}. 

\noindent\textbf{The Virtue of NLL: Coverage for Hard Positives.}
By allocating greater gradient magnitude to low-confidence tokens, NLL ensures that no token is left behind. This \emph{coverage guarantee} is crucial for capturing domain-specific facts or reasoning steps that initially appear unlikely under the pretrained prior. It forces the model to overcome its ignorance, making NLL optimal for acquiring new knowledge. 
\textbf{The Limitation of NLL: Inefficient Sharpening and Noise Overfitting.}
The uniform allocation becomes a liability in two key regimes. 
1) \emph{Inefficient Sharpening at High Confidence:} As the model begins to master the data, the NLL gradient decays linearly with error, providing diminishing returns. 
This tailing behavior slows entropy reduction, preventing the model from fully committing to correct answers even when on the right trajectory.
2) \emph{Vulnerability to Conflicts at Low Confidence:} Since NLL cannot distinguish between learnable gaps and adversarial conflicts, it indiscriminately amplifies gradients in both cases.
In domains where the model already exhibits strong capabilities, low-probability tokens—corresponding to tail noise—can generate large gradients that overwhelm the learning signals from high-quality tokens, thereby weakening robust priors and disrupting the model’s core structure. 
As shown in~\figref{fig:overview}, token-level learning conflicts persist throughout training, as evidenced by the clustering in the upper-left region of the scatter plot and the “Forg” (forgetting) column in the histogram. Notably, the proportion of forgotten tokens increases from steps 0–10 to 0–100, indicating that prolonged forced fitting of conflicting tokens progressively degrades previously established prior knowledge. Moreover, pronounced inter-step (200-210) fluctuations lead to substantial token forgetting, which obstructs effective entropy reduction and ultimately impedes model convergence.

\section{Trust Gating and Entropy Duality}
\label{sec:general_family}
The analysis in ~\secref{sec:dilemma} reveals the need for a confidence-dependent gating mechanism that can adapt learning signal allocation based on the model's current state.
We now establish this reweighting is not merely heuristic but theoretically principled through an optimization--entropy duality. 

\subsection{Trust Gating via the Deformed-Log Family}
\label{subsec:deformed_log}

We formalize a one-parameter family of token losses based on the Tsallis $q$-logarithm~\citep{tsallis1988possible}, which provides the cleanest parameterization of the gate term.

\begin{definition}[$q$-logarithm]
\label{def:qlog}
For $q\in\mathbb{R}$ and $x>0$, the $q$-logarithm is
\begin{equation}
\ln_q(x) \triangleq \frac{x^{1-q}-1}{1-q},
\end{equation}
with the limit $\ln_q(x)\to \ln(x)$ as $q\to 1$.
\end{definition}

\noindent\textbf{Negative $q$-log likelihood.}
For a target-token probability $p\in(0,1]$, we instantiate ~\defref{def:qlog} with $q=1-\alpha$ for a single \emph{focus index} $\alpha\ge 0$ and obtain the token loss in closed form:
\begin{equation}
  \label{eq:example_alpha}
  \mathcal{L}_\alpha(p)=\frac{1-p^\alpha}{\alpha},
\end{equation}
with the continuous limit $\mathcal{L}_\alpha(p)\to -\log p$ as $\alpha\to 0$ (standard NLL).
At $\alpha=1$, $\mathcal{L}_1(p)=1-p$ (linear probability loss); at $\alpha=\tfrac12$, $\mathcal{L}_{1/2}(p)=2(1-\sqrt p)$, proportional to the squared Hellinger distance between a one-hot target and the predicted distribution.
For $0<\alpha<1$, Eq.~\ref{eq:example_alpha} coincides with the generalized cross-entropy (GCE) family studied for robustness to label noise~\citep{zhang2018generalized}.
The $\alpha$ acts as a \emph{focus index}, smoothly interpolating between different learning behaviors:
when $\alpha$ is small, we recover NLL's behavior;
when $\alpha$ is large, the objective increasingly emphasizes high-confidence tokens.

\noindent\textbf{Gradient decomposition.}
For the focus-indexed loss $\mathcal{L}_\alpha(p)=\frac{1-p^\alpha}{\alpha}$, the target-logit gradient is:
\begin{equation}
\frac{\partial\,\mathcal{L}_\alpha(p)}{\partial z_{\tilde y}}
= -\,p^{\alpha}\,(1-p).
\end{equation}
This factorizes into a \emph{trust gate} and a \emph{prediction error}:
\begin{equation}
\label{eq:gate_error}
\underbrace{p^\alpha}_{\text{Gate: }\mathcal{G}(p)}\quad \times\quad \underbrace{(1-p)}_{\text{Error}}.
\end{equation}
The gate term $\mathcal{G}(p)=p^\alpha$ controls how much the model trusts its current prediction: 
When $\alpha\to 0$, the gate is nearly open for all confidence levels, allowing low-confidence tokens to receive full gradient weight.
When $\alpha=1$, the gate scales linearly with confidence, suppressing gradients for low-confidence tokens.
When $\alpha>1$, the gate closes more aggressively for low-confidence tokens, providing stronger emphasis on high-confidence regions.

\subsection{The Optimization-Entropy Duality}
\label{subsec:duality}
The choice of $\alpha$ determines which generalized entropy the learning system preferentially reduces, thereby specifying the \emph{geometric space} in which optimization occurs.

\begin{definition}[Tsallis entropy]
\label{def:tsallis_entropy}
For a discrete distribution $r$ and $q>0$, $q\neq 1$, the Tsallis entropy is
\begin{equation}
S_q(r)\triangleq \frac{1-\sum_y r(y)^q}{q-1},
\end{equation}
with $S_q(r)\to -\sum_y r(y)\log r(y)$ as $q\to 1$.
\end{definition}

\begin{theorem}[Optimization--entropy duality]
\label{thm:opt_entropy_duality}
Fix a token context $c$ and let $r(\cdot\mid c)$ be the true next-token distribution.
For notational convenience within this theorem only, we write $x\equiv c$ in conditioning terms, and we use $q(\cdot\mid x)\equiv \hat p(\cdot\mid x)$ for the predicted distribution.
Define the scoring rule $S_\alpha(\hat p,y)\triangleq \mathcal{L}_\alpha(\hat p(y))=\frac{1-\hat p(y)^\alpha}{\alpha}$.
Then the Bayes risk equals Tsallis entropy of order $1+\alpha$:
\begin{equation}
\begin{aligned}
\min_{\hat p}\ \mathbb{E}_{y \sim r(\cdot \mid c)}\bigl[S_\alpha(\hat p, y)\bigr]
&= \mathbb{E}_{y \sim r(\cdot \mid x)}\bigl[S_\alpha(r, y)\bigr], \\
&= S_{1+\alpha}\bigl(r(\cdot \mid x)\bigr) .
\end{aligned}
\end{equation}
Moreover, for any model $\hat p$,
\begin{equation}
\mathbb{E}_{y\sim r(\cdot\mid x)}\!\left[S_\alpha(\hat p,y)\right]
=\frac{1-\sum_y r(y\mid x)\,q(y\mid x)^\alpha}{\alpha}.
\end{equation}
Thus the loss index is $q_{\mathrm{loss}}=1-\alpha$ while the induced entropy index is $q_{\mathrm{ent}}=1+\alpha$, and
$q_{\mathrm{loss}}+q_{\mathrm{ent}}=2$.
\end{theorem}

\noindent ~\thmref{thm:opt_entropy_duality} reveals that $\alpha$ acts as a control parameter for the \emph{entropy geometry} of optimization.
When $\alpha\to 0$, $q_{\mathrm{ent}}\to 1$ and the induced entropy is Shannon, corresponding to an ``information acquisition'' regime that heavily penalizes low-probability events and ensures coverage at the cost of slow convergence.
When $\alpha\to 1$, $q_{\mathrm{ent}}=2$ and the induced entropy becomes order-$2$ Tsallis (collision) entropy, which emphasizes high-probability mass and accelerates distribution sharpening.
Intermediate values of $\alpha$ target entropy orders strictly between these extremes, balancing exploration and exploitation.
Detailed proofs are provided in \appref{appen:deformed_log}.

\section{Unifying Coverage and Sharpening}
The analysis in ~\secref{sec:general_family} exposes a inherent tension: a single static ($\alpha$) is often brittle—coverage-oriented choices improve learning of hard positives but can overfit conflicts, while sharpening-oriented choices stabilize updates but may under-learn genuine gaps.
This tension necessitates a \emph{state-dependent trajectory} $\alpha(p)$ that automatically modulates the trust gate, transitioning from coverage to sharpening as confidence increases.

\subsection{Geometric Anchors and the Cayley Transform}
\label{subsec:derivation}

We derive $\alpha(p)$ from geometric principles by imposing boundary constraints at the two extremes of confidence.

\noindent \textbf{Boundary Conditions.}
To satisfy the coverage requirement established in ~\secref{subsec:dual_nature}, the objective should recover NLL behavior when the model is uncertain ($p \to 0$). 
This imposes the \emph{Coverage Anchor}: $\alpha(0) = 0$. In this limit, the gate term becomes unity ($s_f(p) \to 1$), ensuring that hard positives receive full gradient signal regardless of initial confidence.
Conversely, to maximize entropy reduction when the model is confident ($p \to 1$), we target the objective to transition to the linear probability loss. 
This imposes the \emph{Sharpening Anchor}: $\alpha(1) = 1$, which maintains maximal gradient pressure for refinement even as $p \to 1$.

\noindent \textbf{Derivation via Conformal Mapping.}
To connect these anchors, we parameterize the model's uncertainty using the intrinsic Fisher--Rao radius $z \triangleq \sqrt{1-p} \in [0, 1]$.
We seek a smooth mapping $\alpha(z)$ that transforms this uncertainty radius $z$ to the focus index $\alpha$, satisfying the boundary conditions: $z=1 \implies \alpha=0$ (max uncertainty $\to$ coverage) and $z=0 \implies \alpha=1$ (certainty $\to$ sharpening).
The canonical conformal mapping that connects these domains on the Riemann sphere is the Cayley transform:
\begin{equation}
\label{eq:Cayley-Transform}
\alpha(z) \;=\; \frac{1-z}{1+z}.
\end{equation}
Substituting $z = \sqrt{1-p}$ yields our state-dependent focus trajectory:
\begin{equation}
\label{eq:alpha_star}
\alpha^*(p) \;=\; \frac{1-\sqrt{1-p}}{1+\sqrt{1-p}}.
\end{equation}
This closed-form solution is parameter-free: near $p\approx 0$, $\alpha^*(p) \approx 0$, prioritizing coverage; as $p$ increases, $\alpha^*(p)$ smoothly ascends to $1$, prioritizing sharpening.

\subsection{Normalized Surprisal to Distribution-Level Surprisal.}
\label{subsec:information_theoretic}

The geometric derivation above possesses a profound information-theoretic interpretation.
By expressing Eq.~\eqref{eq:alpha_star} in terms of the inverse hyperbolic tangent, we reveal that our trajectory is equivalent to:
\begin{equation}
\operatorname{arctanh}(\alpha^*(p)) = -\frac{1}{2}\ln(\sqrt{1-p}) = -\frac{1}{4}\ln(1-p).
\end{equation}
Let $I_{\mathrm{err}} \triangleq -\ln(1-p)$ be the \emph{error surprisal}---the information content of prediction error. We can rewrite $\alpha^*(p)$:
\begin{equation}
\label{eq:alpha_tanh}
\alpha^*(p) \;=\; \tanh\left(\frac{I_{\mathrm{err}}}{4}\right).
\end{equation}
This provides a rigorous justification for our method: the focus index $\alpha$ is modulated by the \emph{normalized error surprisal}. Considering the focus index increases monotonically with the surprisal of the prediction error, smoothly moving from the coverage anchor ($\alpha\approx 0$) toward the sharpening anchor ($\alpha\approx 1$) as the model becomes more confident.

\noindent \textbf{From point-wise surprisal to distribution-level surprisal.}
The Cayley transform establishes an ideal trajectory of 
$a$ as a function of uncertainty $z$. However, another critical limitation of NLL is that it cannot distinguish between “knowledge blind spots unknown to the model” and “conflicting noise in the data.” Using the target probability $p$ alone as a point estimate fails to capture this ambiguity of uncertainty. To address this, we elevate the state estimation from the single-point target probability $p$ to the concentration of the model’s overall distribution (confidence of the model’s overall state).
$P_\theta(\cdot\mid c)=\mathrm{softmax}(z_\theta(c))$.
Specifically, we introduce \textbf{Dynamic Entropy Fine-Tuning (DEFT)} as a parameter-free surrogate that adapts the focus index using a distribution-level concentration signal.

\noindent \textbf{DEFT focus index (Rényi-2 entropy).}
We use the exponential form of the order-2 Rényi entropy as a bounded focus index:
\begin{equation}
\label{eq:deft_alpha}
\alpha_{\mathrm{DEFT}}(c)
\triangleq
C\!\left(P_\theta(\cdot\mid c)\right)
=
\sum_{v\in\cV} P_\theta(v\mid c)^2
\in (0,1].
\end{equation}
When $P_\theta(\cdot\mid c)$ is diffuse (high uncertainty), $\alpha_{\mathrm{DEFT}}(c)$ is small, leading to a more open gate and thus preserving plasticity.
When $P_\theta(\cdot\mid c)$ is concentrated (low uncertainty), $\alpha_{\mathrm{DEFT}}(c)$ approaches $1$, producing a sharpening-style gate and suppressing updates that are inconsistent with a confident model state. Empirically, $\alpha_{\mathrm{DEFT}}(c)$ tracks the same qualitative trend as $\alpha^*(p)$: both increase as the model becomes more certain.
inally, DEFT instantiates the deformed-log gate in Eq.~\eqref{eq:gate_error} by setting $\alpha=\alpha_{\mathrm{DEFT}}(c)$, yielding a token-wise gate
$\mathcal{G}_{\mathrm{DEFT}}(p,c)=p^{\alpha_{\mathrm{DEFT}}(c)}$.
We treat $\alpha_{\mathrm{DEFT}}(c)$ as a stop-gradient quantity to avoid introducing additional curvature into optimization. Detailed proofs are provided in~\appref{appen:resistance}.

\section{Main Experiments}

\begin{table*}[t]
\centering
\caption{Average@16 accuracy under the Model-Strong setting. Best results are in \textbf{bold} and the second-best are \underline{underlined}. In this regime, the negative log-likelihood loss is less effective at promoting probability sharpening, whereas $-p$ generally yield superior performance. Cayley-Trans and DEFT achieve comparable results, which can be attributed to their dynamic optimization mechanisms.}
\scalebox{0.78}{

\begin{tabular}{llcccccc}
\toprule
\textbf{Models} & \textbf{Methods}              & \textbf{Math500}     & \textbf{Minerva Math} & \textbf{Olympiad Bench} & \textbf{AIME24}     & \textbf{AMC23}       & \textbf{Avg}            \\
\midrule

\multirow{6}{*}{\textbf{LLaMA-3.1-8B}} & Base                 & 1.76        & 0.68         & 0.86           & 0.00       & 1.25        & 0.91           \\
& -logp                & 17.58       & 6.15         & 3.11           & 0.00       & 5.63        & 6.49           \\
& -p                   & 23.91       & 9.91         & 5.69           & 0.62       & 10.00       & 10.03          \\
& EAFT                 & 16.84       & 6.20         & 3.44           & 0.00       & 7.34        & 6.76           \\
\rowcolor{cyan!10}
& Cayley-Trans      & 25.99       & 10.46        & 6.18           & 0.21       & 9.06        & \underline{10.38}    \\
\rowcolor{cyan!10}
& DEFT      & 26.26       & 10.21        & 6.43           & 1.03       & 12.81       & \textbf{11.35} \\
\midrule

\multirow{6}{*}{\textbf{DeepseekMath-7B}} & Base                 & 5.70        & 2.89         & 1.51           & 0.00       & 2.34        & 2.49           \\
& -logp                & 28.48       & 9.91         & 6.65           & 0.21       & 8.75        & 10.80          \\
& -p                   & 39.52       & 19.28        & 14.51          & 1.03       & 17.18       & \textbf{18.30} \\
& EAFT                 & 27.70       & 9.26         & 6.96           & 0.21       & 11.09       & 11.04          \\
\rowcolor{cyan!10}
& Cayley-Trans      & 39.03       & 13.78        & 12.26          & 0.41       & 16.72       & 16.44          \\
\rowcolor{cyan!10}
& DEFT      & 40.90       & 15.48        & 12.63          & 1.65       & 20.31       & \underline{18.19}    \\
\midrule

\multirow{6}{*}{\textbf{Qwen2.5-Math-1.5B}} & Base                 & 30.71       & 8.81         & 14.88          & 2.49       & 17.97       & 14.97          \\
& -logp                & 41.70       & 11.81        & 11.39          & 1.45       & 17.66       & 16.80          \\
& -p                   & 63.48       & 25.14        & 26.37          & 6.47       & 39.22       & \textbf{32.14} \\
& EAFT                 & 40.48       & 11.86        & 11.50          & 0.41       & 14.53       & 15.76          \\
\rowcolor{cyan!10}
& Cayley-Trans      & 56.04       & 21.75        & 21.96          & 1.86       & 28.44       & 26.01          \\
\rowcolor{cyan!10}
 & DEFT & \underline{61.41} & \underline{24.13}  & \underline{24.96}    & \underline{5.63} & \underline{35.63} & \underline{30.35}    \\
\midrule

\multirow{6}{*}{\textbf{Qwen2.5-Math-7B}} & Base                 & 40.38       & 13.66        & 16.36          & 6.04       & 24.69       & 20.23          \\
& -logp                & 52.93       & 19.17        & 17.75          & 2.90       & 24.53       & 23.46          \\
& -p                   & 67.98       & 28.46        & 32.12          & 9.38       & 41.25       & \underline{35.84}    \\
& EAFT                 & 51.53       & 20.36        & 17.54          & 2.08       & 25.47       & 23.40          \\
\rowcolor{cyan!10}
& Cayley-Trans      & 66.16       & 31.70        & 30.25          & 7.09       & 39.53       & 34.95          \\
\rowcolor{cyan!10}
& DEFT      & 68.78       & 30.68        & 32.19          & 8.14       & 42.97       & \textbf{36.55} \\
\bottomrule
\end{tabular}
\label{math_table}
}
\end{table*}

This section seeks to empirically verify the effectiveness of the proposed Cayley Transform and the DEFT objective.
Following ~\cite{li2025loglikelihoodprobabilitybasedobjectives}, which categorizes fine-tuning scenarios into Model-Strong, Model-Intermediate, and Model-Weak regimes based on the degree of alignment between a pre-trained model’s prior knowledge, we comprehensively evaluate various loss functions under different levels of prior strength. 
Our experimental design is structured to address the following three key research questions:

\begin{enumerate}
    \item RQ1 (Performance): Can Cayley Transform and DEFT demonstrate superior performance across all regions of the Model Capability Continuum?
    \item RQ2 (Mechanism): Does Cayley Transform and DEFT, as theoretically anticipated, achieve a dynamic balance between conflict suppression and learning knowledge?
    \item RQ3 (Generalization): Does Cayley Transform and DEFT lead to more stable OOD generalization? 
\end{enumerate}

\subsection{Experimental Setup}
In the Model-Strong regime, we adopt the NuminaMath dataset~\cite{numina_math_datasets}. In the Model-Weak regime, we use synthetic FigFont puzzles—a reasoning task based on character art~\cite{stojanovski2025reasoninggymreasoningenvironments}. 
For the model-intermediate regime, we employ the m23K dataset~\cite{huang2025m1}. To more comprehensively evaluate different loss functions when trained on reasoning and general mixed data, we use a subset of Tulu3-SFT~\cite{lambert2024tulu3}. Experiments span a diverse collection of advanced model backbones, including LLaMA-3.2-3B, LLaMA-3.1-8B~\cite{grattafiori2024llama3herdmodels}, DeepSeekMath-7B~\cite{shao2024deepseekmathpushinglimitsmathematical}, Qwen2.5-Math-1.5B, Qwen2.5-Math-7B~\cite{yang2024qwen25mathtechnicalreportmathematical}, Qwen2.5-1.5B, and Qwen2.5-7B~\cite{qwen2025qwen25technicalreport}.
Detailed settings in the~\appref{exp_appendix}.

To evaluate the performance of the Cayley Transform and DEFT and to analyze their underlying mechanisms, we compare them with three representative objectives. $-\log p$ serving as a baseline for undifferentiated full-data learning performance. $-p$ is the linear probability scaling loss, which mitigates noise while preserving the model prior by down-weighting low-probability labels~\cite{li2025beyond, wu2025generalizationsftreinforcementlearning}. EAFT (Entropy-Adaptive Fine-Tuning) incorporates entropy as a regularization term to identify and address low-entropy, low-probability labels~\cite{diao2026entropyadaptivefinetuningresolvingconfident}.

\begin{table*}[t]
\centering
\caption{Main results in the Model-Intermediate setting (\textbf{bold}: best; \underline{underlined}: second-best). In this regime, $-\log p$ (comprehensive coverage) and $-p$ (signal sharpening) yield comparable performance. Notably, DEFT consistently outperforms $-p$ and exceeds $-\log p$ in most cases, demonstrating the superiority of its dynamic mechanism.}
\scalebox{0.72}{
\begin{tabular}{lccccccccccc}
\toprule
\textbf{Model}                   & \textbf{MedMc} & \textbf{MedQA} & \textbf{PubMed} & \textbf{MMLU-P} & \textbf{GPQA}  & \textbf{Lancet}         & \textbf{MedB(4)}        & \textbf{MedB(5)} & \textbf{MedX} & \textbf{NEJM}  & \textbf{Avg}            \\
\midrule
\multicolumn{12}{c}{LLaMA-3.1-8B}                                                                                                                   \\
\midrule
Base                    & 23.57 & 29.14 & 21.00  & 20.00  & 29.49 & 22.57 & 30.52 & 20.45   & 10.01      & 20.73 & 22.75          \\
                       -logp & 55.08 & 60.10 & 72.70  & 52.38  & 36.66 & 62.60 & 46.75 & 43.50   & 15.87      & 56.50 & 50.21          \\
-p & 53.62 & 53.62 & 75.60  & 53.74  & 44.10 & 58.25 & 47.07 & 40.26   & 13.18      & 58.21 & 49.77          \\

                    EAFT    & 55.13 & 61.12 & 75.10  & 54.20  & 34.62 & 57.77 & 53.25 & 45.78   & 16.91      & 57.88 & 51.18          \\
                    \rowcolor{cyan!10}
 Cayley-Trans& 54.63 & 62.14 & 75.10  & 54.59  & 38.46 & 57.52 & 51.62 & 45.78   & 15.73      & 59.20 & \underline{51.48}    \\
 \rowcolor{cyan!10}
DEFT                    & 54.77 & 60.57 & 76.00  & 54.07  & 44.87 & 59.70 & 50.97 & 47.72   & 14.08      & 57.21 & \textbf{52.00} \\
\midrule
\multicolumn{12}{c}{\cellcolor[HTML]{FFFFFF}{\color[HTML]{1F2329} LLaMA-3.2-3B}}                                                                    \\
\midrule
Base                    & 21.30 & 21.92 & 22.60  & 11.40  & 23.08 & 25.00          & 23.05          & 15.26   & 10.35      & 23.22 & 19.72          \\
-logp                   & 42.22 & 48.86 & 66.60  & 35.44  & 32.82 & 44.90          & 45.13          & 37.34   & 12.56      & 44.44 & \textbf{41.03} \\
-p                      & 40.09 & 43.36 & 65.60  & 34.07  & 38.72 & 42.72          & 35.71          & 33.44   & 11.18      & 38.31 & 38.32          \\
EAFT                    & 43.61 & 48.70 & 68.20  & 36.42  & 36.15 & 41.50          & 36.69          & 31.49   & 11.46      & 42.79 & 39.70          \\
\rowcolor{cyan!10}
Cayley-Trans            & 42.67 & 46.43 & 66.80  & 37.65  & 32.56 & 45.63          & 41.56          & 35.71   & 10.63      & 39.30 & 39.89          \\
\rowcolor{cyan!10}
DEFT                    & 42.31 & 46.27 & 66.20  & 36.35  & 33.08 & 47.33          & 39.94          & 33.44   & 11.73      & 42.45 & \underline{39.91}    \\
\midrule
\multicolumn{12}{c}{Qwen2.5-1.5B}                                                                                                                   \\
\midrule
Base                    & 22.21 & 21.84 & 18.50  & 11.21  & 24.36 & 22.57          & 24.03          & 17.53   & 10.84      & 18.74 & 19.18          \\
-logp                   & 38.85 & 40.06 & 68.80  & 35.24  & 32.56 & 36.89          & 36.69          & 28.25   & 9.94       & 32.50 & 35.98          \\
-p                      & 37.39 & 39.75 & 66.90  & 35.70  & 27.69 & 41.75          & 32.79          & 31.17   & 11.73      & 38.97 & 36.38          \\
EAFT                    & 39.85 & 39.91 & 68.90  & 36.03  & 40.77 & 39.08          & 35.06          & 29.87   & 10.70      & 34.49 & \underline{37.47}    \\
\rowcolor{cyan!10}
Cayley-Trans            & 39.23 & 39.75 & 67.10  & 35.44  & 28.72 & 39.56          & 34.42          & 26.30   & 11.25      & 32.17 & 35.39          \\
\rowcolor{cyan!10}
DEFT                    & 37.84 & 41.08 & 69.50  & 35.44  & 34.62 & 38.83          & 39.61          & 32.79   & 12.42      & 32.34 & \textbf{37.45} \\
\midrule
\multicolumn{12}{c}{Qwen2.5-Math-7B}                                                                                                                \\
\midrule
Base                    & 35.84 & 27.26 & 49.30  & 30.23  & 35.90 & 30.34          & 24.03          & 18.18   & 10.21      & 24.71 & 28.60          \\
-logp                   & 36.60 & 34.96 & 71.70  & 36.35  & 35.90 & 37.14          & 27.27          & 29.22   & 11.59      & 26.30 & 34.70          \\
-p                      & 35.48 & 33.54 & 72.50  & 39.15  & 26.92 & 39.32          & 31.49          & 23.05   & 10.77      & 28.03 & 34.03          \\
EAFT                    & 35.98 & 35.19 & 72.90  & 37.39  & 35.38 & 37.14          & 32.47          & 30.19   & 10.63      & 27.36 & \textbf{35.46} \\
\rowcolor{cyan!10}
Cayley-Trans            & 37.49 & 33.62 & 72.00  & 37.52  & 33.59 & 39.56          & 31.82          & 25.65   & 10.84      & 27.03 & \underline{34.91}    \\
\rowcolor{cyan!10}
DEFT                    & 35.96 & 33.70 & 69.50  & 35.60  & 33.59 & 36.65          & 31.17          & 31.17   & 9.52       & 25.90 & 34.28       \\
\bottomrule
\end{tabular}
}
\label{model inter}
\end{table*}




\subsection{Main Results}

\textbf{Model-Strong Results.}~\tabref{math_table} shows the main results in the Model-Strong regime. In this setting, the pre-trained model already encodes rich domain knowledge, and the central optimization challenge is to preserve these robust priors during fine-tuning while adapting to task-specific reasoning patterns and output formats.
The results show that DEFT, $-p$, and Cayley-Trans consistently outperform the standard $-logp$ objective.  
Consistent with prior work \cite{li2025beyond, wu2025generalizationsftreinforcementlearning}, this indicates that in prior-driven mechanisms, suppressing low-probability token gradients while exploiting high-confidence tokens is key to enhancing model performance.
While EAFT mitigates gradients from low-entropy, low-probability tokens, these anomalies are sparse in the Model-Strong regime, resulting in performance comparable to $-logp$. We also report the multi-epoch training dynamics of different methods on Llama3.1-8B in~\tabref{multi_epoch}, where Cayley-Trans and DEFT achieve average performance improvements of 3.62 and 1.46, respectively, compared to NLL.

\begin{table*}[]
\centering
\caption{Main results in the Model-Weak (MW) setting (\textbf{bold}: best; \underline{underlined}: second-best). In this regime, $-p$ struggles to acquire new knowledge due to its diminished gradients for low-probability tokens. DEFT achieves superior overall performance by effectively balancing exploration and exploitation.}
\scalebox{0.9}{
\begin{tabular}{cccccccc}
\toprule
Models                        & Metric                  & Base  & -logp & -p    & EAFT  & Cayley-Trans   & DEFT           \\
\midrule
\multirow{2}{*}{LLaMA-3.1-8B}  & Exact Math              & 0.00     & 0.00  & 0.00  & 0.00  & \underline{0.01}  & \textbf{1.56}  \\
                              & Jaro\_winkler simlarity & 30.17 & 33.40 & 14.26 & \underline{42.71} & 36.95 & \textbf{46.34} \\
                              \midrule
\multirow{2}{*}{LLaMA-3.2-3B}  & Exact Math              & 0.00     & 0.00  & 0.00  & 0.00  & 0.00  & \textbf{3.91}  \\
                              & Jaro\_winkler simlarity & 41.89 & 40.37 & 7.68  & 39.15 & \underline{46.55} & \textbf{60.22} \\
                              \midrule
\multirow{2}{*}{Qwen2.5-1.5B} & Exact Math              & 0.00     & 0.00  & 0.00  & 0.00  & 0.00  & 0.00  \\
                              & Jaro\_winkler simlarity & 35.32 & 33.65 & 28.13 & 31.01 & \textbf{34.21} & \underline{33.66} \\
                              \midrule
\multirow{2}{*}{Qwen2.5-7B}   & Exact Math              & 0.00     & 0.78  & 0.00  &    \textbf{47.66}   & 0.00  & \underline{42.97} \\
                              & Jaro\_winkler simlarity & 44.92 & 40.01 & 27.13 &  \textbf{89.09}     & 34.41 & \underline{86.82} \\
                              \bottomrule
\end{tabular}
}
\label{model weak}
\end{table*}

\begin{table}[]
\centering
\caption{Main results in the Model-Mixed setting using Qwen2.5-1.5B. Best and second-best results are in bold and underlined, respectively. General-Avg and MATH-Avg denote averages over all benchmarks in each category.}
\scalebox{0.9}{
\begin{tabular}{lcccccc}
\toprule
\textbf{Dataset} & \textbf{Base} & \textbf{-logp} & \textbf{-p}    & \textbf{EAFT}  & \textbf{Cayley-Trans} & \textbf{DEFT}           \\
\midrule
General-Avg & 18.80         & 34.79          & 33.37          & 34.52          & 34.34                 & \textbf{35.98}          \\
MATH-Avg    & 4.90          & 9.73           & 13.09          & 10.34          & 11.02                 & \underline{12.09}       \\
\midrule
Total-Avg    & 11.85         & 22.26          & 23.23          & 22.43          & 22.68                 & \textbf{24.04}          \\
\bottomrule
\end{tabular}
}
\label{tab:model_mix}
\end{table}


\begin{figure*}[t]
\centering
\begin{minipage}[b]{\textwidth}
\centering
\includegraphics[width=\textwidth]{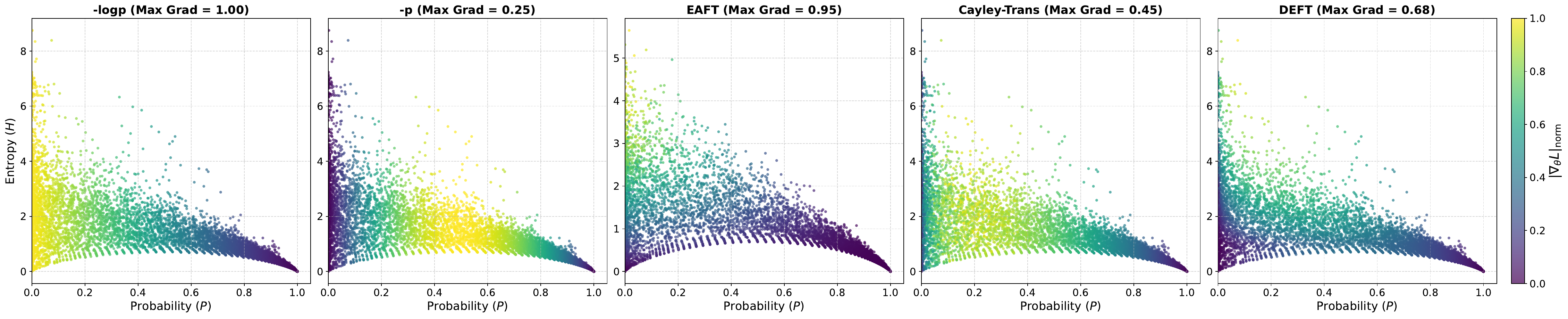} 
\label{fig:a}
\end{minipage}

\begin{minipage}[b]{\textwidth}
\centering
\includegraphics[width=\textwidth]{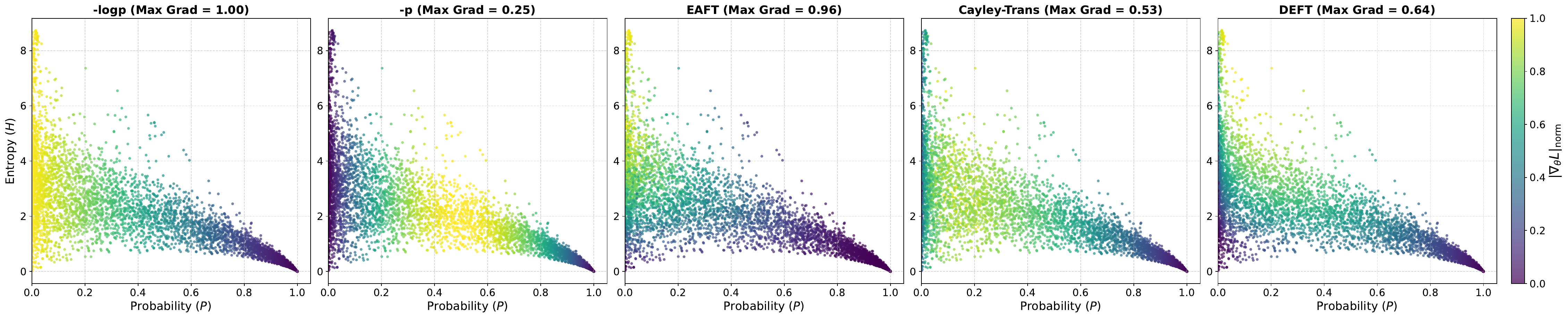} 
\label{fig:b}
\end{minipage}

\begin{minipage}[t]{\textwidth}
\centering
\includegraphics[width=\textwidth]{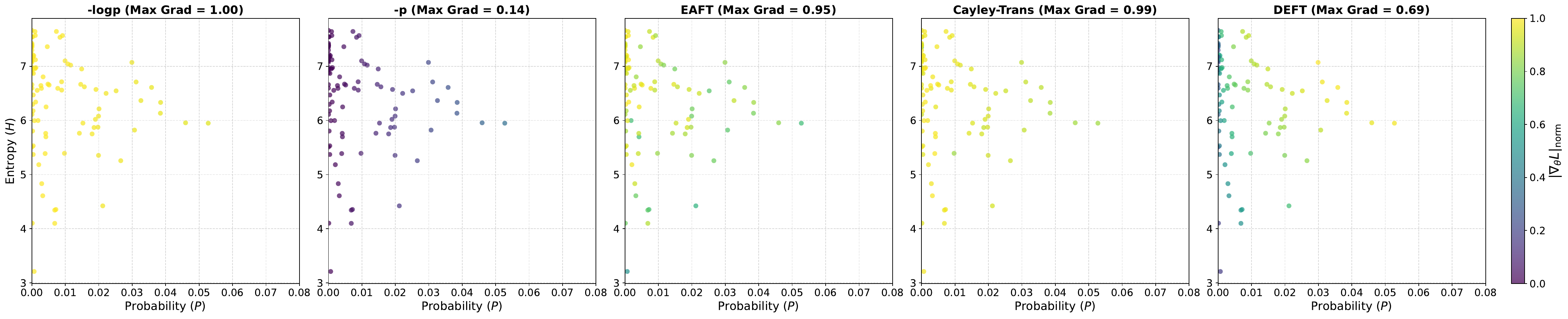} 
\label{fig:c}
\end{minipage}
\caption{
Token-level gradient distributions across model capability regions. The x-axis shows token probability and the y-axis entropy, with color intensity indicating normalized gradient magnitude, illustrating how different loss functions affect token-wise learning. Top: model-strong; Medium: model-medium; bottom: model-weak.}
\label{fig:token_grad}
\end{figure*}

\noindent \textbf{Model-Intermediate Results.}
As shown in Table{model inter}, the performance differences among loss functions are relatively small. Sharpening the model’s prior distribution alone is insufficient for prior-dependent objectives like $-p$ to achieve significant performance gains, and it is also insufficient for prior-averse objectives like $-logp$ to learn decisive content.
Although $-p$ achieves performance comparable to DEFT in the math region, DEFT consistently outperforms $-p$ across all cases in this region. This suggests that under heterogeneous prior strengths, dynamic gradient weighting enables the model to capture more useful information than $-p$.


\noindent \textbf{Model-Weak Results.}
Under the this regime—serving as the benchmark for evaluating learning objective plasticity—the performance gaps among different objectives widen substantially. As shown in~\tabref{model weak}, DEFT significantly outperforms the $-p$ objective and also achieves notable improvements over the $-logp$ baseline. In contrast, the$-p$ objective exhibits severe performance degradation in this regime, indicating its limited capacity to support from-scratch knowledge injection. DEFT and Cayley-Trans, by comparison, effectively overcomes this limitation, demonstrating superior generalization and adaptability.

\noindent \textbf{Mixed Region}
The mixed-domain regime serves as a rigorous testbed for the exploration-exploitation trade-off. As shown in~\tabref{tab:model_mix}, while $-logp$ falters on reasoning tasks due to overfitting and $-p$ sacrifices general capability, DEFT achieves an optimal equilibrium. By dynamically modulating gradients, DEFT secures the highest score on General tasks and dominates the overall average, yielding substantial gains over $-logp$ (+8.0\%) and $-p$ (+3.5\%). These results confirm that entropy-driven adaptive gating effectively navigates complex distributions where static objectives fail.  Detailed results are reported in~\tabref{mix_general} and~\tabref{mix_math}.

\subsection{Mechanism Analysis}
\textbf{Gradient Landscape and Objective Analysis.}
To reveal the mechanisms underlying different objectives, we visualize the token-level gradient distributions in the Model-Strong and Model-Weak regimes, as shown in~\figref{fig:token_grad}. The gradient distribution for the Model-Intermediate regime is presented in the Appendix.
Distinct gradient allocation patterns emerge across methods. Under NLL, gradients are dominated by low-probability long-tail tokens. This distribution compels the model to overemphasize potential noise or anomalies, creating a risk that overfitting to rare tokens will distort the learned probability manifold.
Conversely, $-p$ exhibits a pronounced head-biased behavior, concentrating gradient mass in medium- to high-probability regions. While this effectively filters noise, it compromises the learning of "hard" positive examples.

\noindent Although EAFT effectively suppresses gradients from low-entropy, low-probability tokens, it retains NLL-like aggressive update trajectories in the rest of the probability space.
In contrast, the dynamic gradient gating methods, the gradient magnitudes in dynamic gating methods such as DEFT and Cayley-Trans are not determined solely by token probability $p$, but are modulated by the model’s internal uncertainty. As a result, model updates are not dominated by low-probability tokens, effectively alleviating the marginal effects of NLL.
In low-probability regions—typically indicating deterministic conflicts or chaotic model states—gradients are scaled according to the model’s internal state. DEFT further mitigates overfitting to low-probability, low-entropy tokens, effectively preserving the model’s core structure while minimizing the impact of harmful parameter shifts, all while maintaining high learning capacity.

\begin{figure*}[t]
\centering
\includegraphics[width=\textwidth]{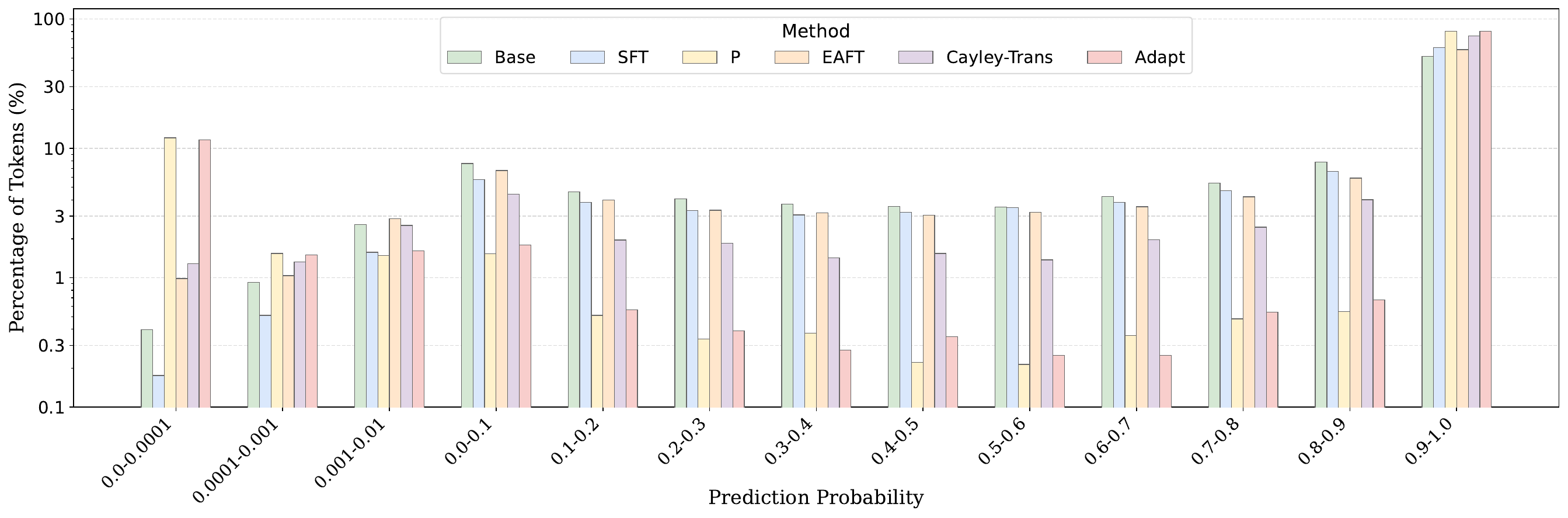}
\caption{Token probability distributions on the training set under different objectives in the Model-Stong regime.}
\label{fig:token_prob_dist}
\end{figure*}

\begin{figure*}[t]
    \centering
    \includegraphics[width=\textwidth]{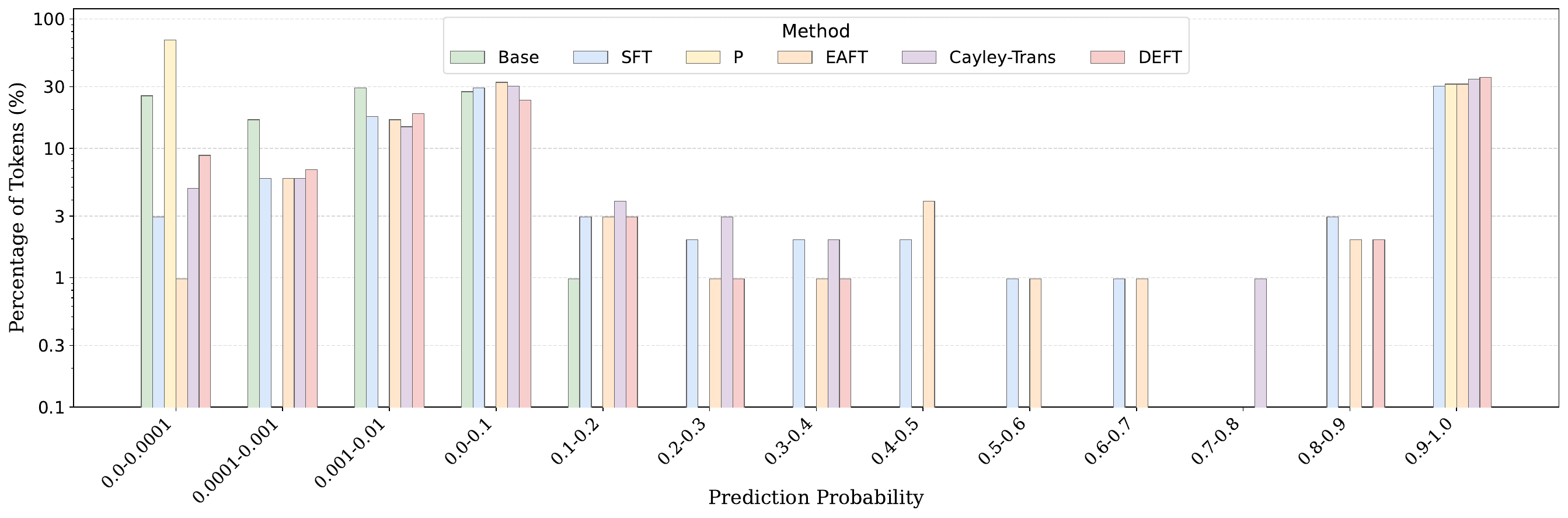}
    \caption{Token probability distributions on the training set under different objectives in the Model-Weak regime.}
    \label{fig:token_prob_dist_figfont}
\end{figure*}

\begin{figure}[htbp]
    \centering
    \begin{minipage}{0.48\textwidth}
        \centering
        \includegraphics[width=\textwidth]{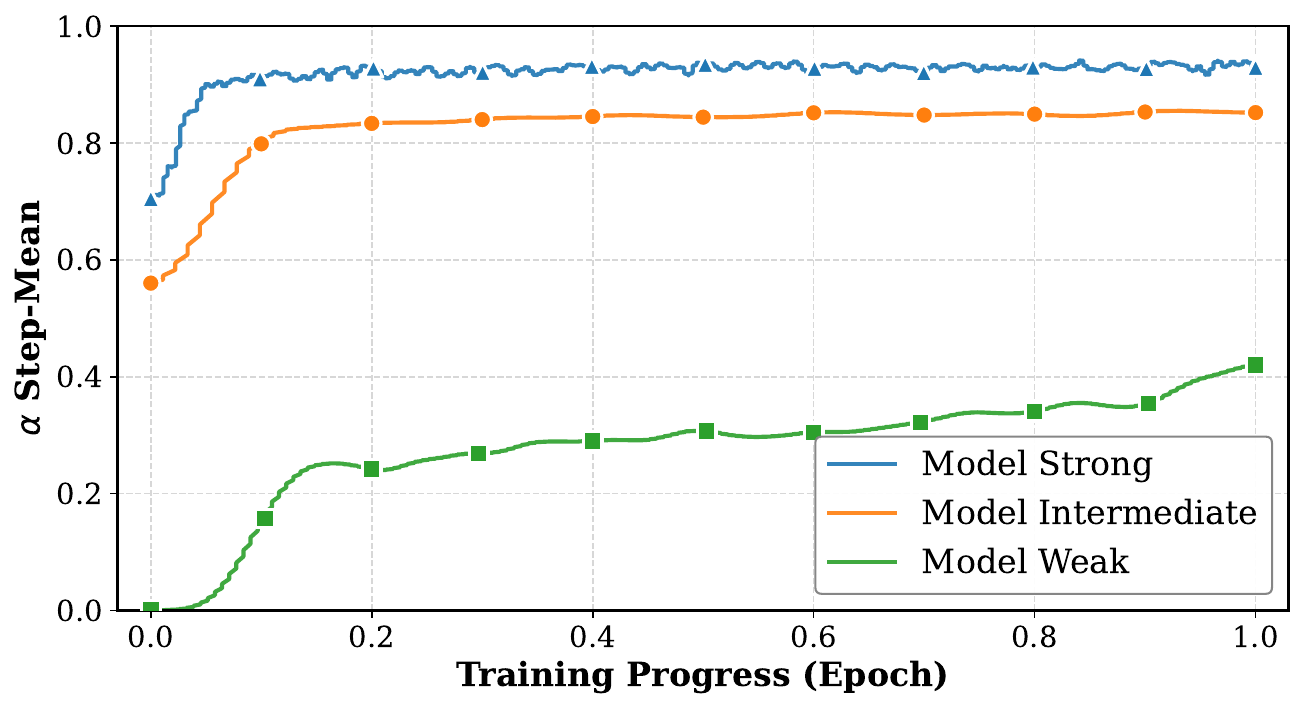}
    \end{minipage}
    \hfill
    \begin{minipage}{0.48\textwidth}
        \centering
        \includegraphics[width=\textwidth]{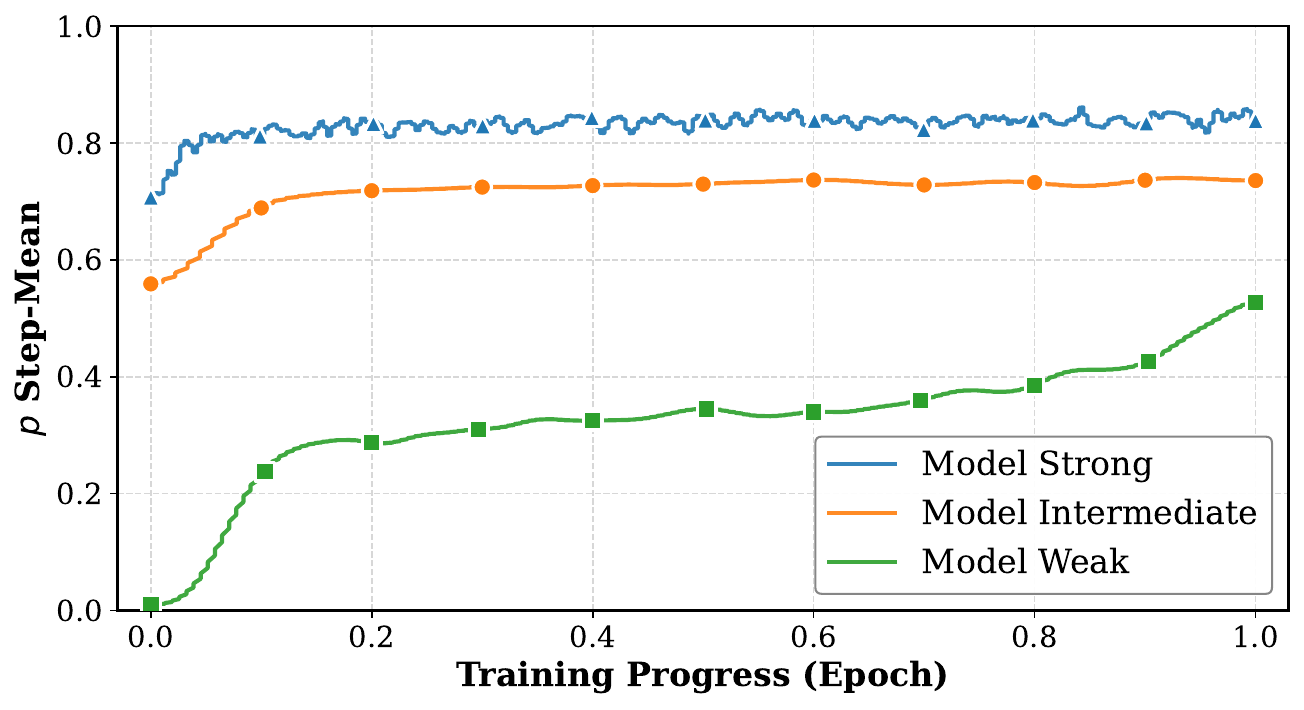}
    \end{minipage}
\caption{Evolution of trust gate $a$ and average probability $p$ during Llama-3.1-8B training with DEFT. $a$ exhibits regime-specific initialization (e.g., 0.65 in Model-Strong vs.\ 0.05 in Model-Weak), reflecting DEFT's adaptive response to the model's varying levels of prior knowledge. Furthermore, as training progresses, 
$a$ increases steadily, facilitating a smooth transition from a coverage-oriented phase to a sharpening-oriented phase.}
    \label{fig:alpha_train}
\end{figure}

\begin{table*}[t]
\centering
\caption{Evaluation of Llama-3.1-8B on Medical benchmarks after fine-tuning on MATH. Best results are in bold.}
\scalebox{0.75}{

\begin{tabular}{llllllllllll}
\toprule
Method       & MedMc & MedQA & PubMed & MedX & MMLU-P & GPQA  & Lancet & MedB(4) & MedB(5) & NEJM  & Avg    
\\
\midrule
-logp        & 41.45 & 50.51 & 71.40  & 13.18      & 43.26  & 42.56 & 40.29  & 43.83   & 36.36   & 48.59 & 43.14          \\
-p           & 46.31 & 54.05 & 74.20  & 12.63      & 43.06  & 35.90 & 47.82  & 41.88   & 33.12   & 51.58 & 44.06          \\
EAFT         & 45.92  & 52.08  & 74.00   & 12.22       & 48.40   & 39.49  & 47.33   & 43.51    & 36.69    & 47.60  & 41.92           \\
\rowcolor{cyan!10}
CayLey-Trans & 45.49 & 53.02 & 77.50  & 13.46      & 45.67  & 41.03 & 48.54  & 41.88   & 32.79   & 49.25 & {\underline{44.86} }    \\
\rowcolor{cyan!10}
DEFT         & 45.90 & 53.57 & 75.60  & 13.94      & 46.06  & 43.59 & 45.87  & 45.13   & 36.04   & 47.76 & \textbf{45.35} \\
\bottomrule
\end{tabular}
\label{ood}
}

\end{table*}

\noindent \textbf{Model Output Distributions Across Model Regimes.} 
To investigate how different training objectives shape the model’s output distribution, we visualize the token probability distributions on the training set.
\figref{fig:token_prob_dist}--\figref{fig:token_prob_dist_figfont} reveal the post-training token probability distributions under various loss functions across two extreme regimes: Model-Strong and Model-Weak. The distribution in the Model-Intermediate regime resembles that of Model-Strong, as illustrated in ~\figref{fig:token_prob_dist_medical}.
\noindent In Model-Strong regime, such as MATH, low-probability tokens often correspond to non-semantic functional elements \citep{wu2025generalizationsftreinforcementlearning}.
Aggressively fitting these tokens, as done with NLL, can lead to forgetting of many high-probability prior tokens, making it difficult to substantially increase the proportion of high-probability tokens as shown in~\figref{fig:overview}. 
Cayley-Trans serves as a smoother alternative to NLL. It shifts probability mass to high-confidence regions without inflating extremely low-probability tokens, thereby mitigating the tail-noise overfitting observed in NLL. 
DEFT further optimizes this balance, retaining the strict noise resistance of $-p$ suppressing $(p<10^{-3})$ while aggressively sharpening high-confidence predictions $(p>0.5)$. 
Conversely, the model weak regime offers another perspective for evaluating the plasticity of learning objectives as shown in~\figref{fig:token_prob_dist_figfont}. Since the majority of tokens are low probability, the $-p$ objective suffers from vanishing gradients and fails to facilitate effective learning. In contrast, DEFT and Cayley-Trans adaptively align their behavior with NLL for the majority of tokens to ensure knowledge acquisition. Moreover, by continuously reinforcing high-confidence knowledge, DEFT achieves superior performance under this setting.
\noindent \textbf{Adaptive Balancing between Prior Sharpening and New Knowledge Acquisition.} To gain a deeper understanding of DEFT’s dynamic gating mechanism, we analyze the evolution of the trust gate $a$ and the average predicted probability $p$ of target tokens during the training of Llama-3.1-8B.  As illustrated in~\figref{fig:alpha_train}, $a$ exhibits distinct initialization and growth patterns across different regimes. In the Model-Strong regime (e.g., MATH), where the model possesses robust prior knowledge, $a$ initializes at a high value of 0.65, rapidly ascends to 0.75 within the first 4 steps, and eventually stabilizes near 0.95. Conversely, in the Model-Weak regime, the sparse initial knowledge state results in a much lower starting $a$ of 0.04, which gradually climbs to 0.45 as learning progresses. Notably, the trend of $a$ closely mirrors the trajectory of the mean target token probability throughout the training process. This dynamic gating enables the model to smoothly transition from "coverage-oriented" behavior (resembling NLL) in unfamiliar domains to "sharpening-oriented" behavior (resembling $-p$) in prior-aligned domains. By adaptively balancing exploration and exploitation in this manner, DEFT achieves superior overall performance across diverse tasks.
\subsection{Generalization}
To rigorously evaluate the robustness of representations learned under different loss functions, we extended our analysis to OOD scenarios. Specifically, we assessed the zero-shot performance of LLaMA-3.1-8B fine-tuned on the NuminaMath dataset when transferred to the distinct medical domain. This setup represents a significant distributional shift, serving as a critical stress test for catastrophic forgetting. As shown in~\tabref{ood}, Cayley-Trans and DEFT demonstrate superior cross-domain generalization, which indicate that the adaptive objectives not only fit the target distribution effectively but also preserve the model's pre-trained core competencies against catastrophic forgetting.

\section{Conclusion}
This work introduces a unified deformed-log framework to address the suboptimality of uniform gradient allocation in SFT. We propose Dynamic Entropy Fine-Tuning (DEFT), which leverages a Cayley-transform-based gating mechanism to dynamically balance sharpening of strong priors with the acquisition of new knowledge. Extensive experiments and analyses demonstrate that DEFT achieves a superior exploration–exploitation trade-off

\bibliography{main}

\newpage
\appendix




\section{Setup}
\label{appen:notation}

For a prompt context $c\in\cC$, let $\tilde{y}\in\cV$ denote the supervised target token, where $\cV$ is the vocabulary.
The token-level training distribution is denoted by $T$, which generates pairs $(c,\tilde{y})$.
We write $r(\cdot\mid c)\in\Delta^{|\cV|-1}$ for the true conditional next-token distribution.
Let $z_{\mtheta}(c)\in\R^{|\cV|}$ denote the pre-softmax logits of an autoregressive language model with parameters $\mtheta$.
The predicted next-token distribution is
\begin{equation}
\label{eq:app:1}
p_{\mtheta}(\cdot\mid c)=\softmax\!\bigl(z_{\mtheta}(c)\bigr)\in\Delta^{|\cV|-1}.
\end{equation}
For brevity, we write $p_{\mtheta}(y\mid c)$ as the probability assigned to token $y$ under context $c$.
We denote the base model by $p_{\mtheta_0}$.
For a differentiable function $f:[0,1]\to\R$, the general SFT objective is
\begin{equation}
\label{eq:app:2}
\cL_f(\mtheta)\;=\;\E_{(c,\tilde{y})\sim T}\!\left[f\!\left(p_{\mtheta}(\tilde{y}\mid c)\right)\right].
\end{equation}
The standard negative log-likelihood (NLL) objective corresponds to $f(p)=-\log p$.
For a fixed context $c$ and supervised target $\tilde{y}$, we abbreviate
\begin{equation}
\label{eq:app:3}
p\;\triangleq\; p_{\mtheta}(\tilde{y}\mid c)\in(0,1].
\end{equation}
Throughout this appendix, we use $e_y\in\{0,1\}^{|\cV|}$ to denote the one-hot vector for token $y$,
and $\mathbf{1}\in\R^{|\cV|}$ to denote the all-ones vector.

\section{Gradient Structure}
\label{appen:gradient_structure}

\subsection{Detailed Derivation of the Softmax-Logit Gradient}

\begin{lemma}[Softmax Jacobian]
\label{lem:softmax_jacobian}
Let $z\in\R^{|\cV|}$ be a vector of logits and $p=\softmax(z)\in\Delta^{|\cV|-1}$ be the corresponding probabilities. Then for any $i,j\in\cV$,
\begin{equation}
\label{eq:app:4}
\frac{\partial p_i}{\partial z_j}
\;=\;
p_i\,(\delta_{i,j}-p_j),
\end{equation}
where $\delta_{i,j}$ is the Kronecker delta ($\delta_{i,j}=1$ if $i=j$ and $\delta_{i,j}=0$ otherwise).
\end{lemma}

\begin{proof}
\textbf{Step 0: Write the softmax formula explicitly.}
By definition,
\begin{equation}
\label{eq:app:5}
p_i=\frac{\exp(z_i)}{\sum_{k=1}^{|\cV|}\exp(z_k)}.
\end{equation}
Define the partition function $Z\coloneqq \sum_{k=1}^{|\cV|}\exp(z_k)$, so that $p_i=\exp(z_i)/Z$.

\textbf{Step 1: Apply the quotient rule.}
To compute $\partial p_i/\partial z_j$, we use the quotient rule:
\begin{equation}
\label{eq:app:6}
\frac{\partial p_i}{\partial z_j}
=\frac{\partial}{\partial z_j}\!\left(\frac{\exp(z_i)}{Z}\right)
=\frac{(\partial_j\exp(z_i))\,Z - \exp(z_i)\,(\partial_j Z)}{Z^2}.
\end{equation}

\textbf{Step 2: Evaluate the numerator and denominator separately.}
\begin{itemize}
\item $\partial_j\exp(z_i)=\delta_{i,j}\exp(z_i)$.
\item $\partial_j Z=\partial_j\sum_{k}\exp(z_k)=\exp(z_j)$.
\end{itemize}
Substituting these into the quotient rule:
\begin{equation}
\label{eq:app:7}
\frac{\partial p_i}{\partial z_j}
=\frac{\delta_{i,j}\exp(z_i)\,Z - \exp(z_i)\,\exp(z_j)}{Z^2}
=\frac{\delta_{i,j}\exp(z_i)}{Z} - \frac{\exp(z_i)\,\exp(z_j)}{Z^2}.
\end{equation}

\textbf{Step 3: Simplify using the softmax probabilities.}
Recall that $p_i=\exp(z_i)/Z$ and $p_j=\exp(z_j)/Z$. Thus,
\begin{equation}
\label{eq:app:8}
\frac{\partial p_i}{\partial z_j}
=\delta_{i,j}\,p_i - p_i p_j
=p_i\,(\delta_{i,j}-p_j),
\end{equation}
as claimed.
\end{proof}

\begin{lemma}[General Objective Gradient]
\label{lem:logit_grad_full}
Let $f:[0,1]\to\R$ be differentiable and nonincreasing.
Consider the objective in Eq.~\ref{eq:general_obj}, whose step-$t$ contribution depends on the correct-class probability $p_t\triangleq p_{\mtheta}(\tilde{y}\mid c)$ only through $f(p_t)$.
Then the gradient of $\cL_f$ with respect to the logits $z_{t,i}$ satisfies:
\begin{equation}
\label{eq:app:9}
\frac{\partial \cL_f}{\partial z_{t,i}}
\;=\;
s_f\!\left(p_t\right)\,\bigl(p_{t,i} - \delta_{i,\tilde{y}}\bigr),
\qquad
\text{where}\quad
s_f(p)\;\triangleq\; -\,f'(p)\,p\;\ge 0.
\end{equation}
In particular, for the target token $i=\tilde{y}$,
\begin{equation}
\label{eq:app:10}
\frac{\partial \cL_f}{\partial z_{t,\tilde{y}}}
= -s_f(p_t)\,(1-p_t)
\;=:\;
-W_f(p_t),
\qquad
W_f(p)\;\triangleq\; s_f(p)\,(1-p).
\end{equation}
\end{lemma}

\begin{proof}
\textbf{Step 0: Express the loss contribution at step $t$.}
The loss contribution for the pair $(c,\tilde{y})$ is $f(p_t)$, where $p_t=p_{\mtheta}(\tilde{y}\mid c)$.

\textbf{Step 1: Apply the chain rule through softmax.}
By the chain rule,
\begin{equation}
\label{eq:app:11}
\frac{\partial \cL_f}{\partial z_{t,i}}
=\frac{\partial f(p_t)}{\partial p_t}\cdot\frac{\partial p_t}{\partial z_{t,i}}.
\end{equation}

\textbf{Step 2: Compute each factor separately.}
\begin{itemize}
\item The derivative of $f$ with respect to $p$ is simply $f'(p_t)$.
\item From ~\lemref{lem:softmax_jacobian}, setting the output neuron index to $\tilde{y}$,
\begin{equation}
\label{eq:app:12}
\frac{\partial p_t}{\partial z_{t,i}}
=\frac{\partial p_{\tilde{y}}}{\partial z_{t,i}}
=p_{\tilde{y}}\,(\delta_{i,\tilde{y}}-p_i)
=p_t\,(\delta_{i,\tilde{y}}-p_{t,i}).
\end{equation}
\end{itemize}

\textbf{Step 3: Combine and rearrange.}
Multiplying the two factors:
\begin{equation}
\label{eq:app:13}
\frac{\partial \cL_f}{\partial z_{t,i}}
=f'(p_t)\,p_t\,(\delta_{i,\tilde{y}}-p_{t,i})
=\bigl(-f'(p_t)\,p_t\bigr)\,(p_{t,i}-\delta_{i,\tilde{y}}).
\end{equation}

\textbf{Step 4: Define the gate term.}
Let $s_f(p)\triangleq -f'(p)\,p$. Since $f$ is nonincreasing by assumption, $f'(p)\le 0$ for all $p\in(0,1)$, hence $s_f(p)\ge 0$.
This yields
\begin{equation}
\label{eq:app:14}
\frac{\partial \cL_f}{\partial z_{t,i}}
=s_f(p_t)\,(p_{t,i}-\delta_{i,\tilde{y}}).
\end{equation}

\textbf{Step 5: Specialize to the target logit $i=\tilde{y}$.}
Setting $i=\tilde{y}$ and using $\delta_{\tilde{y},\tilde{y}}=1$:
\begin{equation}
\label{eq:app:15}
\frac{\partial \cL_f}{\partial z_{t,\tilde{y}}}
=s_f(p_t)\,(p_t-1)
=-s_f(p_t)\,(1-p_t).
\end{equation}
Define the learning signal magnitude $W_f(p)\triangleq s_f(p)(1-p)=-f'(p)\,p(1-p)$, which yields
\begin{equation}
\label{eq:app:16}
\frac{\partial \cL_f}{\partial z_{t,\tilde{y}}}=-W_f(p_t),
\end{equation}
as claimed.
\end{proof}


\section{The Deformed-Log Family and Entropy Duality}
\label{appen:deformed_log}

\subsection{Properties of the $q$-Logarithm}

\begin{proposition}[Continuity and limit properties of $\ln_q$]
\label{prop:qlog_properties}
For the $q$-logarithm defined by
\begin{equation}
\label{eq:app:17}
\ln_q(x)\triangleq \frac{x^{1-q}-1}{1-q},
\end{equation}
we have the following properties:
\begin{enumerate}
\item As $q\to 1$, $\ln_q(x)\to \ln(x)$ for any $x>0$.
\item For fixed $q\ne 1$, $\ln_q$ is strictly monotone increasing in $x>0$.
\item The derivative is $\frac{\dm}{\dm x}\ln_q(x)=x^{-q}$.
\end{enumerate}
\end{proposition}

\begin{proof}
\textbf{(1) Limit as $q\to 1$ using L'Hôpital's rule.}

Fix $x>0$. Write $\ln_q(x)=\frac{x^{1-q}-1}{1-q}$. As $q\to 1$, both numerator and denominator approach zero:
\begin{equation}
\label{eq:app:18}
\lim_{q\to 1}(x^{1-q}-1)=x^0-1=0,\qquad \lim_{q\to 1}(1-q)=0.
\end{equation}
This is an indeterminate form $0/0$, so we apply L'Hôpital's rule.

\textbf{Step 1: Differentiate numerator and denominator with respect to $q$.}
\begin{itemize}
\item Numerator: $\frac{\dm}{\dm q}(x^{1-q})=\frac{\dm}{\dm q}\exp\!\bigl((1-q)\ln x\bigr)=-\ln x\cdot x^{1-q}$.
\item Denominator: $\frac{\dm}{\dm q}(1-q)=-1$.
\end{itemize}

\textbf{Step 2: Apply L'Hôpital's rule.}
\begin{equation}
\label{eq:app:19}
\lim_{q\to 1}\ln_q(x)
=\lim_{q\to 1}\frac{-\ln x\cdot x^{1-q}}{-1}
=\ln x\cdot\lim_{q\to 1}x^{1-q}
=\ln x\cdot x^0
=\ln x.
\end{equation}

\textbf{(2) Monotonicity.}

Take the derivative with respect to $x$:
\begin{equation}
\label{eq:app:20}
\frac{\dm}{\dm x}\ln_q(x)
=\frac{\dm}{\dm x}\!\left(\frac{x^{1-q}-1}{1-q}\right)
=\frac{1}{1-q}\,(1-q)x^{(1-q)-1}
=x^{-q}.
\end{equation}
Since $x>0$ and $x^{-q}>0$ for any $q$, $\ln_q(x)$ is strictly increasing in $x$.

\textbf{(3) Derivative formula.}

This was computed in step (2):
\begin{equation}
\label{eq:app:21}
\frac{\dm}{\dm x}\ln_q(x)=x^{-q}.
\end{equation}
\end{proof}

\subsection{Optimization-Entropy Duality}

\begin{theorem}[Optimization-Entropy Duality, restating ~\thmref{thm:opt_entropy_duality}]
\label{thm:opt_entropy_duality_full}
Fix a token context $c$ and let $r(\cdot\mid c)$ be the true next-token distribution.
For notational convenience within this theorem, we write $x\equiv c$ in conditioning terms, and we use $\hat p(\cdot\mid x)$ for the predicted distribution.
Define the Tsallis scoring rule:
\begin{equation}
S_\alpha(\hat p,y)
\triangleq \frac{1-(1+\alpha)\hat p(y)^\alpha+\alpha\sum_{y'}\hat p(y'\mid x)^{1+\alpha}}{\alpha}.
\end{equation}
Then the Bayes risk equals the Tsallis entropy of order $1+\alpha$:
\begin{equation}
\label{eq:app:22}
\min_{\hat{p}}\ \E_{y\sim r(\cdot\mid c)}\!\left[S_\alpha(\hat{p},y)\right]
= \E_{y\sim r(\cdot\mid x)}\!\left[S_\alpha(r,y)\right]
= S_{1+\alpha}\!\left(r(\cdot\mid x)\right).
\end{equation}
Moreover, for any model $\hat{p}$,
\begin{equation}
\label{eq:app:23}
\E_{y\sim r(\cdot\mid x)}\!\left[S_\alpha(\hat{p},y)\right]
= \frac{1}{\alpha}
 - \frac{1+\alpha}{\alpha}\sum_y r(y\mid x)\,\hat p(y\mid x)^\alpha
 + \sum_y \hat p(y\mid x)^{1+\alpha}.
\end{equation}
Thus the induced entropy index is $q_{\mathrm{ent}}=1+\alpha$; under the common $q$-log parameterization with $q_{\mathrm{loss}}=1-\alpha$, we have $q_{\mathrm{loss}}+q_{\mathrm{ent}}=2$.
\end{theorem}

\begin{proof}
\textbf{Step 0: State the problem formally.}

We want to minimize the expected score
\begin{equation}
\label{eq:app:24}
\cR(\hat{p})\;\triangleq\;\E_{y\sim r(\cdot\mid x)}\!\left[S_\alpha(\hat p,y)\right]
\end{equation}
over all probability distributions $\hat{p}\in\Delta^{|\cV|-1}$.
This is a constrained optimization problem: minimize $\cR(\hat{p})$ subject to $\sum_y \hat{p}(y)=1$ and $\hat{p}(y)\ge 0$ for all $y$.

\textbf{Step 1: Write the Lagrangian.}

Introduce a Lagrange multiplier $\lambda$ for the normalization constraint $\sum_y \hat{p}(y)=1$.
The Lagrangian is
\begin{equation}
\label{eq:app:25}
\cL(\hat{p},\lambda)
=\cR(\hat p) - \lambda\!\left(\sum_y \hat{p}(y)-1\right).
\end{equation}
Expanding the expectation:
\begin{equation}
\label{eq:app:26}
\cL(\hat{p},\lambda)
=\frac{1}{\alpha}
 - \frac{1+\alpha}{\alpha}\sum_y r(y)\,\hat p(y)^\alpha
 + \sum_y \hat p(y)^{1+\alpha}
 - \lambda\!\left(\sum_y \hat{p}(y)-1\right).
\end{equation}

\textbf{Step 2: Take the first-order condition for each $\hat{p}(y)$.}

For each $y\in\cV$, differentiate $\cL$ with respect to $\hat{p}(y)$:
\begin{equation}
\label{eq:app:27}
\frac{\partial \cL}{\partial \hat{p}(y)}
=-(1+\alpha)\,r(y)\,\hat{p}(y)^{\alpha-1} + (1+\alpha)\,\hat{p}(y)^{\alpha} - \lambda.
\end{equation}
Setting this to zero:
\begin{equation}
\label{eq:app:28}
\hat{p}(y)^{\alpha-1}\bigl(\hat{p}(y)-r(y)\bigr)=\frac{\lambda}{1+\alpha}.
\end{equation}

\textbf{Step 3: Conclude $\hat p^*(y)=r(y)$.}

Since $\hat p(y)^{\alpha-1}\ge 0$, the sign of the left-hand side in Eq.~\ref{eq:app:28} matches the sign of $\hat p(y)-r(y)$.
Because the right-hand side is the same constant for all $y$, it follows that $\hat p(y)-r(y)$ must have the same sign for all $y$.
But $\sum_y (\hat p(y)-r(y))=0$, so the only possibility is $\hat p(y)=r(y)$ for all $y$.
Thus the unique minimizer is $\hat p^*=r$.

\textbf{Step 4: Evaluate the Bayes risk at $\hat p^*=r$.}

Substituting $\hat{p}^*(y)=r(y)$ into Eq.~\ref{eq:app:23} yields
\begin{equation}
\label{eq:app:32}
\cR(\hat{p}^*)
=\frac{1}{\alpha} - \frac{1+\alpha}{\alpha}\sum_y r(y)^{1+\alpha} + \sum_y r(y)^{1+\alpha}
=\frac{1-\sum_y r(y)^{1+\alpha}}{\alpha}.
\end{equation}
\textbf{Step 5: Recognize the Tsallis entropy.}

Recall the Tsallis entropy of order $q$ is defined as
\begin{equation}
\label{eq:app:34}
S_q(r)\triangleq \frac{1-\sum_y r(y)^q}{q-1}.
\end{equation}
Setting $q=1+\alpha$:
\begin{equation}
\label{eq:app:35}
S_{1+\alpha}(r)
=\frac{1-\sum_y r(y)^{1+\alpha}}{(1+\alpha)-1}
=\frac{1-\sum_y r(y)^{1+\alpha}}{\alpha}.
\end{equation}
This exactly matches $\cR(\hat{p}^*)$. Therefore,
\begin{equation}
\label{eq:app:36}
\min_{\hat{p}}\;\cR(\hat{p})
=\cR(\hat{p}^*)
=S_{1+\alpha}(r).
\end{equation}

\textbf{Step 6: Evaluate the risk for a general $\hat p$.}

For any model $\hat p(\cdot\mid x)$, the expected score is
\begin{equation}
\label{eq:app:37}
\E_{y\sim r}\!\left[S_\alpha(\hat p,y)\right]
= \frac{1}{\alpha}
 - \frac{1+\alpha}{\alpha}\sum_y r(y)\,\hat p(y)^\alpha
 + \sum_y \hat p(y)^{1+\alpha}.
\end{equation}

\textbf{Step 7: Relate the loss index and entropy index.}

In the $q$-logarithm parameterization, we have $\alpha=1-q_{\mathrm{loss}}$.
The induced Tsallis entropy order is $q_{\mathrm{ent}}=1+\alpha=1+(1-q_{\mathrm{loss}})=2-q_{\mathrm{loss}}$.
Thus,
\begin{equation}
\label{eq:app:38}
q_{\mathrm{loss}}+q_{\mathrm{ent}}=q_{\mathrm{loss}}+(2-q_{\mathrm{loss}})=2,
\end{equation}
as claimed.
\end{proof}


\section{Unifying Coverage and Sharpening}
\label{appen:resistance}

\subsection{Geometric Derivation of the State-Dependent Focus Trajectory}

In this section, we provide a detailed geometric derivation of the state-dependent focus index $\alpha^*(p)$ that automatically transitions from coverage (NLL) to sharpening (linear loss) as the model's confidence increases.

\begin{definition}[Fisher-Rao radius]
\label{def:fisher_rao_radius}
For a probability $p\in(0,1]$, define the intrinsic uncertainty radius as
\begin{equation}
\label{eq:app:39}
z\;\triangleq\;\sqrt{1-p}\in[0,1].
\end{equation}
This quantity is a convenient \emph{monotone re-parameterization} of the Fisher-Rao distance-to-certainty for Bernoulli distributions (made explicit below).
\end{definition}

\begin{proposition}[Fisher-Rao distance to certainty for Bernoulli]
\label{prop:fr_distance_bernoulli}
Consider the Bernoulli family $\mathrm{Bern}(p)$ with parameter $p\in(0,1)$.
The Fisher-Rao metric is
\begin{equation}
\label{eq:app:39b}
\mathrm{d}s^2=\frac{\mathrm{d}p^2}{p(1-p)}.
\end{equation}
The Fisher-Rao geodesic distance from $\mathrm{Bern}(p)$ to certainty $\mathrm{Bern}(1)$ is
\begin{equation}
\label{eq:app:39c}
d_{\mathrm{FR}}(p,1)
\;=\;\int_{p}^{1}\frac{\mathrm{d}u}{\sqrt{u(1-u)}}
\;=\;2\arccos\!\bigl(\sqrt{p}\bigr).
\end{equation}
Moreover, the radius $z=\sqrt{1-p}$ equals the sine of the half-distance:
\begin{equation}
\label{eq:app:39d}
z=\sqrt{1-p}=\sin\!\left(\frac{d_{\mathrm{FR}}(p,1)}{2}\right).
\end{equation}
\end{proposition}

\begin{proof}
\textbf{Step 0: Fisher information and Fisher-Rao metric.}
For Bernoulli, the Fisher information is $I(p)=\frac{1}{p(1-p)}$, which induces the 1D Fisher-Rao line element $\mathrm{d}s^2=I(p)\,\mathrm{d}p^2$ and yields Eq.~\ref{eq:app:39b}.

\textbf{Step 1: Integrate the line element.}
The geodesic distance in 1D is the integral of the line element:
\[
d_{\mathrm{FR}}(p,1)=\int_p^1\frac{\mathrm{d}u}{\sqrt{u(1-u)}}.
\]
Substitute $u=\cos^2\theta$ with $\theta\in[0,\tfrac{\pi}{2}]$. 
Then $\mathrm{d}u=-2\cos\theta\sin\theta\,\mathrm{d}\theta$ and $\sqrt{u(1-u)}=\cos\theta\sin\theta$, hence
\[
\int_p^1\frac{\mathrm{d}u}{\sqrt{u(1-u)}}
=
\int_{\arccos(\sqrt{p})}^{0}\frac{-2\cos\theta\sin\theta\,\mathrm{d}\theta}{\cos\theta\sin\theta}
=
2\arccos(\sqrt{p}),
\]
which gives Eq.~\ref{eq:app:39c}.

\textbf{Step 2: Relate $z$ to $d_{\mathrm{FR}}$.}
From Eq.~\ref{eq:app:39c}, we have $d_{\mathrm{FR}}(p,1)/2=\arccos(\sqrt{p})$.
Thus,
\[
\sin\!\left(\frac{d_{\mathrm{FR}}(p,1)}{2}\right)=\sin\!\bigl(\arccos(\sqrt{p})\bigr)=\sqrt{1-p}=z,
\]
which is Eq.~\ref{eq:app:39d}.
\end{proof}

\paragraph{From confidence $p$ to intrinsic radius $z$.}
We will construct the focus index as a function of the Fisher-Rao radius $z$ in Eq.~\ref{eq:app:39}, and then pull it back to $p$ via $z=\sqrt{1-p}$.
This is conceptually convenient because $z$ is an \emph{intrinsic} coordinate: by Eqs.~\ref{eq:app:39c}--\eqref{eq:app:39d}, it equals $\sin\!\bigl(d_{\mathrm{FR}}(p,1)/2\bigr)$, i.e., a monotone re-parameterization of the Fisher-Rao distance to certainty.
Thus, specifying $\alpha(z)$ amounts to specifying how the objective should transition as the model approaches certainty along the information-geometric geodesic.

\begin{proposition}[Boundary conditions for coverage and sharpening]
\label{prop:boundary_conditions}
To unify coverage and sharpening, we require a mapping $\alpha:[0,1]\to[0,1]$ satisfying:
\begin{enumerate}
\item \textbf{Coverage anchor:} $\alpha(0)=0$ (when $p\to 0$, recover NLL behavior).
\item \textbf{Sharpening anchor:} $\alpha(1)=1$ (when $p\to 1$, recover linear loss).
\end{enumerate}
\end{proposition}

\begin{proof}[Justification]
\textbf{Setup: an interpolating gate family.}
In this section we consider the common interpolating \emph{gate} ansatz
\begin{equation}
\label{eq:app:39e}
s(p)=p^{\alpha},
\qquad \alpha\in[0,1],
\end{equation}
which continuously connects the NLL gate ($s(p)=1$ when $\alpha=0$) and the linear-probability gate ($s(p)=p$ when $\alpha=1$).
We then allow $\alpha=\alpha(p)$ to depend on the current confidence $p$.

\textbf{(1) Coverage anchor.}

When the model assigns very low probability $p\to 0$ to the supervised target, we want the gate to be \emph{fully open} so that hard positives receive essentially the full gradient magnitude.
This corresponds to the NLL behavior: for $f(p)=-\ln p$ we have $f'(p)=-1/p$ and hence the general gate
$s_f(p)=-f'(p)\,p$ in Eq.~\ref{eq:app:9} becomes $s_f(p)\equiv 1$.
This is recovered by requiring $\alpha(p)\to 0$ as $p\to 0$ so that $p^{\alpha(p)}\to 1$ (made precise in Lemma~\ref{lem:gate_limit}).

\textbf{(2) Sharpening anchor.}

When the model is highly confident ($p\to 1$), we want to maximize the gradient pressure to refine the distribution further.
This corresponds to the linear-probability loss, e.g. $f(p)=1-p$, for which $f'(p)=-1$ and Eq.~\ref{eq:app:9} gives $s_f(p)=p$.
\end{proof}

\begin{lemma}[When does $p^{\alpha(p)}\to 1$ as $p\to 0$?]
\label{lem:gate_limit}
Let $\alpha:(0,1]\to[0,1]$ satisfy $\lim_{p\to 0}\alpha(p)=0$ and $\lim_{p\to 0}\alpha(p)\,|\ln p|=0$. Then
\begin{equation}
\label{eq:app:39f}
\lim_{p\to 0}p^{\alpha(p)}=1.
\end{equation}
\end{lemma}

\begin{proof}
Take logs:
\[
\ln\bigl(p^{\alpha(p)}\bigr)=\alpha(p)\ln p.
\]
Since $\ln p\to-\infty$ as $p\to 0$, the assumption $\alpha(p)\,|\ln p|\to 0$ implies $\alpha(p)\ln p\to 0$, hence $p^{\alpha(p)}\to e^0=1$.
\end{proof}

\begin{theorem}[The Cayley Transform Solution]
\label{thm:cayley_solution}
Among Möbius (linear-fractional) maps $\alpha:[0,1]\to[0,1]$ that swap the endpoints
\begin{equation}
\label{eq:app:40}
z=1\;\Rightarrow\;\alpha=0,\qquad z=0\;\Rightarrow\;\alpha=1,
\end{equation}
the general form is the one-parameter family $\alpha_\kappa(z)=\frac{1-z}{1+\kappa z}$ with $\kappa>-1$, and every such map is \emph{self-dual} in the sense of being an involution,
\begin{equation}
\label{eq:app:40b}
\alpha(\alpha(z))=z,
\end{equation}
in particular including the affine reflection ($\kappa=0$) and a continuum of non-affine choices ($\kappa\ne 0$).
If we further impose the \emph{surprisal linearization} requirement that $\operatorname{arctanh}(\alpha_\kappa(z))$ is affine in $\ln z$ (equivalently, compatibility with the normalized-surprisal form in Prop.~\ref{prop:surprisal_gating} without any additive offset), this uniquely selects $\kappa=1$, i.e. the Cayley transform:
\begin{equation}
\label{eq:app:41}
\alpha(z)\;=\;\frac{1-z}{1+z}.
\end{equation}
Substituting $z=\sqrt{1-p}$ yields the state-dependent focus trajectory:
\begin{equation}
\label{eq:app:42}
\alpha^*(p)\;=\;\frac{1-\sqrt{1-p}}{1+\sqrt{1-p}}.
\end{equation}
\end{theorem}

\begin{proof}
\textbf{Step 0: Write the general Möbius form.}

A Möbius transformation has the form
\begin{equation}
\label{eq:app:43}
\alpha(z)=\frac{az+b}{cz+d},\qquad ad-bc\ne 0.
\end{equation}
\textbf{Step 1: Impose the endpoint constraints.}
We need to find $a,b,c,d$ such that:
\begin{enumerate}
\item $\alpha(1)=0$: $\frac{a(1)+b}{c(1)+d}=0 \Rightarrow a+b=0 \Rightarrow b=-a$.
\item $\alpha(0)=1$: $\frac{a(0)+b}{c(0)+d}=1 \Rightarrow \frac{b}{d}=1 \Rightarrow b=d$.
\end{enumerate}

\textbf{Step 2: Reduce to a one-parameter family.}
From $b=-a$ and $b=d$, we have $d=-a$. Plugging into Eq.~\ref{eq:app:43} gives
\[
\alpha(z)=\frac{az-a}{cz-a}=\frac{z-1}{(\tfrac{c}{a})z-1}.
\]
Let $\kappa\coloneqq -\tfrac{c}{a}$ (this re-parameterization is without loss of generality since scaling $(a,b,c,d)$ leaves $\alpha$ unchanged). Then
\begin{equation}
\label{eq:app:45}
\alpha(z)=\frac{1-z}{1+\kappa z}.
\end{equation}
For any $\kappa>-1$, this map is well-defined on $z\in[0,1]$ and is monotone decreasing since
\begin{equation}
\label{eq:app:45c}
\alpha'(z)=-\frac{1+\kappa}{(1+\kappa z)^2}\le 0.
\end{equation}

\textbf{Step 3: Impose self-duality (involution).}
We additionally require Eq.~\ref{eq:app:40b}, i.e., $\alpha(\alpha(z))=z$ for all $z\in[0,1]$.
Substituting Eq.~\ref{eq:app:45} into itself yields
\[
\alpha(\alpha(z))
=
\frac{1-\frac{1-z}{1+\kappa z}}{1+\kappa\frac{1-z}{1+\kappa z}}
=
\frac{(1+\kappa)z/(1+\kappa z)}{(1+\kappa)/(1+\kappa z)}
=z.
\]
Thus, the involution constraint does \emph{not} further restrict $\kappa$ beyond the domain/range condition $\kappa>-1$.
To select a canonical non-affine trajectory, we impose an additional information-theoretic regularity: $\operatorname{arctanh}(\alpha_\kappa(z))$ should be affine in $\ln z$ (equivalently, the normalized-surprisal representation in Prop.~\ref{prop:surprisal_gating} should have no additive offset).
Using Eq.~\ref{eq:app:49} and $\alpha_\kappa(z)=\frac{1-z}{1+\kappa z}$, we obtain
\begin{equation}
\label{eq:app:41b}
\operatorname{arctanh}(\alpha_\kappa(z))
=\frac{1}{2}\ln\!\left(\frac{2+(\kappa-1)z}{(\kappa+1)z}\right).
\end{equation}
This expression is affine in $\ln z$ if and only if $2+(\kappa-1)z$ is constant in $z$, i.e. $\kappa=1$, which yields the Cayley transform in Eq.~\ref{eq:app:41}.

\textbf{Step 4: Substitute $z=\sqrt{1-p}$ to obtain $\alpha^*(p)$.}

Using the Fisher-Rao radius $z=\sqrt{1-p}$:
\begin{equation}
\label{eq:app:42_restate}
\alpha^*(p)
=\alpha\!\bigl(\sqrt{1-p}\bigr)
=\frac{1-\sqrt{1-p}}{1+\sqrt{1-p}}.
\end{equation}

\textbf{Step 5: Verify the endpoint behavior.}
\begin{itemize}
\item As $p\to 0$: $z=\sqrt{1-0}=1$, so $\alpha^*(0)=\frac{1-1}{1+1}=0$. 
\item As $p\to 1$: $z=\sqrt{1-1}=0$, so $\alpha^*(1)=\frac{1-0}{1+0}=1$. 
\end{itemize}
\end{proof}

\begin{proposition}[Asymptotics of $\alpha^*(p)$]
\label{prop:alpha_asymptotics}
The Cayley trajectory satisfies the following expansions:
\begin{equation}
\label{eq:app:47b}
\alpha^*(p)=\frac{p}{4}+O(p^2)\quad\text{as }p\to 0,
\qquad
\alpha^*(p)=1-2\sqrt{1-p}+O(1-p)\quad\text{as }p\to 1.
\end{equation}
In particular, $\alpha^*(p)\in[0,1]$ and is strictly increasing in $p\in(0,1)$.
\end{proposition}

\begin{proof}
For $p\to 0$, use $\sqrt{1-p}=1-\tfrac{p}{2}-\tfrac{p^2}{8}+O(p^3)$ and substitute into Eq.~\ref{eq:app:42}.
For $p\to 1$, let $\epsilon=1-p$ so $\sqrt{1-p}=\sqrt{\epsilon}$ and expand Eq.~\ref{eq:app:42} as $\epsilon\to 0$.
Monotonicity follows by differentiating Eq.~\ref{eq:app:42}:
\[
\frac{\mathrm{d}\alpha^*}{\mathrm{d}p}
=
\frac{1}{\sqrt{1-p}\,\bigl(1+\sqrt{1-p}\bigr)^2}>0.
\]
\end{proof}

\begin{corollary}[Endpoint-consistent gate limits for $s(p)=p^{\alpha^*(p)}$]
\label{cor:gate_limits_alpha_star}
Define $s^*(p)\coloneqq p^{\alpha^*(p)}$. Then
\begin{equation}
\label{eq:app:47c}
\lim_{p\to 0}s^*(p)=1,
\qquad
\lim_{p\to 1}\frac{s^*(p)}{p}=1.
\end{equation}
\end{corollary}

\begin{proof}
As $p\to 0$, Prop.~\ref{prop:alpha_asymptotics} gives $\alpha^*(p)=O(p)$, hence $\alpha^*(p)|\ln p|\to 0$ and Lemma~\ref{lem:gate_limit} implies $p^{\alpha^*(p)}\to 1$.
As $p\to 1$, we have $\alpha^*(p)\to 1$, so $s^*(p)=p^{\alpha^*(p)}=p\cdot p^{\alpha^*(p)-1}$ and $p^{\alpha^*(p)-1}\to 1$.
\end{proof}

\subsection{Information-Theoretic Interpretation via Normalized Surprisal}

\begin{proposition}[Normalized Surprisal Gating]
\label{prop:surprisal_gating}
The Cayley trajectory $\alpha^*(p)$ admits an information-theoretic interpretation in terms of normalized error surprisal:
\begin{equation}
\label{eq:app:48}
\alpha^*(p)\;=\;\tanh\!\left(\frac{I_{\mathrm{err}}}{4}\right),
\qquad
I_{\mathrm{err}}\;\triangleq\;-\ln(1-p).
\end{equation}
\end{proposition}

\begin{proof}
\textbf{Step 0: Recall the inverse hyperbolic tangent.}

The inverse hyperbolic tangent is defined by
\begin{equation}
\label{eq:app:49}
\operatorname{arctanh}(x)=\frac{1}{2}\ln\!\left(\frac{1+x}{1-x}\right),\qquad x\in(-1,1).
\end{equation}

\textbf{Step 1: Express $\alpha^*(p)$ in terms of $\operatorname{arctanh}$.}

From \thmref{thm:cayley_solution}, we have
\begin{equation}
\label{eq:app:50}
\alpha^*(p)=\frac{1-\sqrt{1-p}}{1+\sqrt{1-p}}.
\end{equation}
Let $u\coloneqq\sqrt{1-p}$. Then $\alpha^*(p)=\frac{1-u}{1+u}$.

Taking $\operatorname{arctanh}$ of both sides:
\begin{equation}
\label{eq:app:51}
\operatorname{arctanh}(\alpha^*)
=\operatorname{arctanh}\!\left(\frac{1-u}{1+u}\right)
=\frac{1}{2}\ln\!\left(\frac{1+(1-u)/(1+u)}{1-(1-u)/(1+u)}\right).
\end{equation}

\textbf{Step 2: Simplify the argument of $\ln$.}

Compute the numerator:
\begin{equation}
\label{eq:app:52}
1+\frac{1-u}{1+u}=\frac{(1+u)+(1-u)}{1+u}=\frac{2}{1+u}.
\end{equation}
Compute the denominator:
\begin{equation}
\label{eq:app:53}
1-\frac{1-u}{1+u}=\frac{(1+u)-(1-u)}{1+u}=\frac{2u}{1+u}.
\end{equation}
Thus,
\begin{equation}
\label{eq:app:54}
\frac{1+(1-u)/(1+u)}{1-(1-u)/(1+u)}
=\frac{2/(1+u)}{2u/(1+u)}
=\frac{2}{2u}
=\frac{1}{u}.
\end{equation}

\textbf{Step 3: Substitute back into $\operatorname{arctanh}$.}
\begin{equation}
\label{eq:app:55}
\operatorname{arctanh}(\alpha^*)
=\frac{1}{2}\ln\!\left(\frac{1}{u}\right)
=-\frac{1}{2}\ln(u)
=-\frac{1}{2}\ln\!\bigl(\sqrt{1-p}\bigr)
=-\frac{1}{4}\ln(1-p).
\end{equation}

\textbf{Step 4: Define the error surprisal and invert.}

Let $I_{\mathrm{err}}\triangleq -\ln(1-p)$ be the information content of a prediction error.
Then:
\begin{equation}
\label{eq:app:56}
\operatorname{arctanh}(\alpha^*)=\frac{I_{\mathrm{err}}}{4}
\quad\Rightarrow\quad
\alpha^*(p)=\tanh\!\left(\frac{I_{\mathrm{err}}}{4}\right).
\end{equation}
\end{proof}


\subsection{Formal Properties of DEFT}
\label{appen:deft_formal}

We recall that for a token context $c$, the model predictive distribution is
$P_\theta(\cdot\mid c)\in\Delta^{|\cV|-1}$, and DEFT defines the focus index
\begin{equation}
\alpha_{\mathrm{DEFT}}(c)\;\triangleq\; \sum_{v\in\cV} P_\theta(v\mid c)^2 .
\end{equation}
The DEFT gate for the supervised target-token probability $p\in(0,1]$ is
$\mathcal{G}_{\mathrm{DEFT}}(p,c)=p^{\alpha_{\mathrm{DEFT}}(c)}$.

\begin{proposition}[Range and extremizers of $\alpha_{\mathrm{DEFT}}$]
\label{prop:deft_range}
For any distribution $P\in\Delta^{|\cV|-1}$,
\begin{equation}
\frac{1}{|\cV|}\;\le\;\sum_{v\in\cV}P(v)^2\;\le\;1.
\end{equation}
Moreover, $\sum_v P(v)^2=1$ if $P$ is a point mass (one-hot), and
$\sum_v P(v)^2=1/|\cV|$ if $P$ is uniform.
\end{proposition}

\begin{proof}
The upper bound follows from
$$
\Bigl(\sum_{v}P(v)\Bigr)^2=\sum_v P(v)^2 + 2\sum_{u<v}P(u)P(v)\;\ge\;\sum_v P(v)^2,
$$
and since $\sum_v P(v)=1$, we get $\sum_v P(v)^2\le 1$. Equality holds iff all
cross terms vanish, i.e., $P(u)P(v)=0$ for all $u\neq v$, which means $P$ is a
point mass.

For the lower bound, apply Cauchy--Schwarz to vectors $(1,\dots,1)$ and
$(P(1),\dots,P(|\cV|))$:
$$
\Bigl(\sum_v 1\cdot P(v)\Bigr)^2 \le \Bigl(\sum_v 1^2\Bigr)\Bigl(\sum_v P(v)^2\Bigr)
\quad\Rightarrow\quad
1 \le |\cV|\sum_v P(v)^2.
$$
Thus $\sum_v P(v)^2\ge 1/|\cV|$, with equality iff $P$ is proportional to the
all-ones vector, i.e., uniform.
\end{proof}

\begin{proposition}[Equivalence to Rényi-2 (collision) entropy]
\label{prop:deft_renyi2}
Let $H_2(P)$ denote the order-2 Rényi entropy:
\begin{equation}
H_2(P)\;\triangleq\;-\ln\sum_{v\in\cV} P(v)^2.
\end{equation}
Then $\alpha_{\mathrm{DEFT}}(c)=\exp\!\bigl(-H_2(P_\theta(\cdot\mid c))\bigr)$.
In particular, $\alpha_{\mathrm{DEFT}}(c)$ is a strictly decreasing (one-to-one)
function of $H_2$ on $\Delta^{|\cV|-1}$.
\end{proposition}

\begin{proof}
By definition,
$$
\exp(-H_2(P))=\exp\!\Bigl(\ln\sum_v P(v)^2\Bigr)=\sum_v P(v)^2.
$$
Since the exponential is strictly monotone and $-\ln(\cdot)$ is strictly
decreasing on $(0,\infty)$, the mapping $P\mapsto \alpha_{\mathrm{DEFT}}(P)$ is
strictly decreasing in $H_2(P)$.
\end{proof}

\begin{proposition}[Gate interpolation between coverage and sharpening]
\label{prop:deft_gate_interpolation}
Fix any $p\in(0,1]$ and define $g(\alpha)=p^\alpha$ for $\alpha\in[0,1]$. Then:
\begin{equation}
g(0)=1,\qquad g(1)=p,\qquad \text{and } g(\alpha)\text{ is nonincreasing in }\alpha.
\end{equation}
Consequently, for any context $c$,
\begin{equation}
p \;\le\; \mathcal{G}_{\mathrm{DEFT}}(p,c)=p^{\alpha_{\mathrm{DEFT}}(c)} \;\le\; 1,
\end{equation}
so the DEFT gate continuously interpolates between an ``open'' gate (coverage)
and a ``sharpening'' gate.
\end{proposition}

\begin{proof}
We have $g(0)=p^0=1$ and $g(1)=p^1=p$. For monotonicity, for $p\in(0,1)$,
$$
\frac{d}{d\alpha}p^\alpha = p^\alpha \ln p \le 0,
$$
since $\ln p\le 0$. Hence $g$ is nonincreasing. Plugging
$\alpha=\alpha_{\mathrm{DEFT}}(c)\in[1/|\cV|,1]$ from
Prop.~\ref{prop:deft_range} yields $p\le p^\alpha\le 1$.
\end{proof}

\begin{proposition}[General decomposition and bounds linking $\alpha_{\mathrm{DEFT}}$ to target probability]
\label{prop:deft_general_link_to_p}
Fix a context $c$ and supervised target token $\tilde y$. Let
$p \triangleq P_\theta(\tilde y\mid c)\in(0,1]$.
Define the normalized \emph{tail distribution} $\bar P_\theta(\cdot\mid c)\in\Delta^{|\cV|-2}$ over $\cV\setminus\{\tilde y\}$ by
\begin{equation}
\bar P_\theta(v\mid c)\;\triangleq\;
\frac{P_\theta(v\mid c)}{1-p},\qquad v\neq \tilde y
\quad (\text{when }p<1).
\end{equation}
Let the tail concentration be
\begin{equation}
S_{\mathrm{tail}}(c)\;\triangleq\;\sum_{v\neq \tilde y}\bar P_\theta(v\mid c)^2 \in \left[\frac{1}{|\cV|-1},\,1\right].
\end{equation}
Then for all $p\in(0,1)$,
\begin{equation}
\label{eq:deft_general_decomp}
\alpha_{\mathrm{DEFT}}(c)
=\sum_{v\in\cV}P_\theta(v\mid c)^2
=
p^2+(1-p)^2\,S_{\mathrm{tail}}(c).
\end{equation}
Consequently,
\begin{equation}
\label{eq:deft_general_bounds}
p^2+\frac{(1-p)^2}{|\cV|-1}
\;\le\;
\alpha_{\mathrm{DEFT}}(c)
\;\le\;
p^2+(1-p)^2.
\end{equation}
\end{proposition}

\begin{proof}
For $p\in(0,1)$, split the sum into target and non-target parts:
$$
\alpha_{\mathrm{DEFT}}(c)=p^2+\sum_{v\neq \tilde y}P_\theta(v\mid c)^2
=
p^2+(1-p)^2\sum_{v\neq \tilde y}\left(\frac{P_\theta(v\mid c)}{1-p}\right)^2
=
p^2+(1-p)^2 S_{\mathrm{tail}}(c),
$$
which proves Eq.~\eqref{eq:deft_general_decomp}.

It remains to bound $S_{\mathrm{tail}}(c)$. Since $\bar P_\theta(\cdot\mid c)$ is a probability distribution over $|\cV|-1$ tokens, by the same argument as in Prop.~\ref{prop:deft_range},
$$
\frac{1}{|\cV|-1}\le \sum_{v\neq \tilde y}\bar P_\theta(v\mid c)^2 \le 1.
$$
Substituting these bounds into Eq.~\eqref{eq:deft_general_decomp} yields Eq.~\eqref{eq:deft_general_bounds}.
\end{proof}

\begin{remark}[Connection to the stylized ``spike + uniform tail'' case and Cayley]
\label{rem:deft_cayley_connection}
Prop.~\ref{prop:deft_general_link_to_p} implies that DEFT can be decomposed as
$$
\alpha_{\mathrm{DEFT}}(c)=p^2+(1-p)^2 S_{\mathrm{tail}}(c),
$$
where $S_{\mathrm{tail}}(c)$ captures the \emph{shape} (concentration) of the
non-target mass. A commonly used stylized model is the ``spike + uniform tail''
assumption, i.e., $\bar P_\theta(\cdot\mid c)$ is uniform over $\cV\setminus\{\tilde y\}$,
which corresponds to $S_{\mathrm{tail}}(c)=1/(|\cV|-1)$. In this special case,
$$
\alpha_{\mathrm{DEFT}}(c)
=
p^2+\frac{(1-p)^2}{|\cV|-1},
$$
which is increasing in $p$ for $p\ge 1/|\cV|$.
Since the Cayley trajectory $\alpha^*(p)$ is strictly increasing in $p$ on $(0,1)$,
both focus indices move in the same direction as confidence grows (trend consistency)
under this stylized regime.
\end{remark}

\begin{proposition}[Conflict suppression under confidently misaligned predictive states]
\label{prop:deft_conflict_suppression}
Fix a context $c$ and supervised target token $\tilde y$, and denote
$p\triangleq P_\theta(\tilde y\mid c)\in(0,1)$.
Assume the predictive distribution is \emph{confidently misaligned} in the sense that
there exists a non-target token $v^\star\neq \tilde y$ such that
\begin{equation}
P_\theta(v^\star\mid c)\ge 1-\varepsilon
\qquad \text{for some } \varepsilon\in(0,1).
\end{equation}
Then the DEFT focus index satisfies $\alpha_{\mathrm{DEFT}}(c)\ge (1-\varepsilon)^2$, and the
corresponding target-logit learning-signal magnitude
\begin{equation}
W_{\mathrm{DEFT}}(p,c)\;\triangleq\;p^{\alpha_{\mathrm{DEFT}}(c)}(1-p)
\end{equation}
is upper bounded as
\begin{equation}
W_{\mathrm{DEFT}}(p,c)\;\le\; p^{(1-\varepsilon)^2}(1-p),
\end{equation}
which vanishes as $p\to 0$.

\noindent In contrast, for the Cayley trajectory $\alpha^*(p)$ in Theorem~\ref{thm:cayley_solution},
the Cayley-gated signal $W_{\mathrm{Cayley}}(p)\triangleq p^{\alpha^*(p)}(1-p)$ satisfies
\begin{equation}
\lim_{p\to 0} W_{\mathrm{Cayley}}(p)=1,
\end{equation}
\end{proposition}

\begin{proof}
\textbf{Step 1: DEFT suppresses the learning signal under severe misalignment.}
By definition, $\alpha_{\mathrm{DEFT}}(c)=\sum_v P_\theta(v\mid c)^2$. Since all terms are nonnegative,
$$
\alpha_{\mathrm{DEFT}}(c)\ge P_\theta(v^\star\mid c)^2 \ge (1-\varepsilon)^2.
$$
For any fixed $p\in(0,1)$, the map $a\mapsto p^a$ is strictly decreasing because
$\frac{d}{da}p^a=p^a\ln p<0$. Therefore,
$$
p^{\alpha_{\mathrm{DEFT}}(c)} \le p^{(1-\varepsilon)^2},
$$
and multiplying by $(1-p)$ yields
$$
W_{\mathrm{DEFT}}(p,c)=p^{\alpha_{\mathrm{DEFT}}(c)}(1-p)\le p^{(1-\varepsilon)^2}(1-p)\xrightarrow[p\to 0]{}0.
$$

\noindent \textbf{Step 2: Cayley gating remains unsuppressed as $p\to 0$.}
The Cayley trajectory has the closed form
$$
\alpha^*(p)=\frac{1-\sqrt{1-p}}{1+\sqrt{1-p}}.
$$
Using the expansion $\sqrt{1-p}=1-\frac{p}{2}+O(p^2)$ as $p\to 0$, we obtain
$$
\alpha^*(p)=\frac{p}{4}+O(p^2).
$$
Hence $\alpha^*(p)\ln p \to 0$ as $p\to 0$, and thus
$$
\lim_{p\to 0} p^{\alpha^*(p)}
=\lim_{p\to 0}\exp\!\bigl(\alpha^*(p)\ln p\bigr)
=\exp(0)=1.
$$
Since $(1-p)\to 1$, it follows that $W_{\mathrm{Cayley}}(p)=p^{\alpha^*(p)}(1-p)\to 1$.
\end{proof}




\section{Additional Theoretical Results}
\label{appen:additional_theory}

\subsection{Convexity Properties and Peak Location}

\begin{proposition}[Convex vs. Concave Objectives]
\label{prop:convex_concave_full}
Let $f\in C^2[0,1]$ with $f'(p)<0$ for all $p\in(0,1)$, and define $W_f(p)=-f'(p)\,p(1-p)$.
\begin{enumerate}
\item If $f$ is concave ($f''\le 0$), then any maximizer of $W_f$ lies in $[\tfrac{1}{2},1]$.
\item If $f$ is convex ($f''\ge 0$), then any maximizer of $W_f$ lies in $[0,\tfrac{1}{2}]$.
\end{enumerate}
\end{proposition}

\begin{proof}
\textbf{Step 0: Set up the problem.}

Define $s(p)\coloneqq -f'(p)>0$. Then
\begin{equation}
\label{eq:app:57}
W_f(p)=s(p)\,p(1-p).
\end{equation}

\textbf{Step 1: Compute the first derivative of $W_f$.}

Using the product rule:
\begin{equation}
\label{eq:app:58}
W_f'(p)
=s'(p)\,p(1-p) + s(p)\,\frac{\dm}{\dm p}\!\bigl[p(1-p)\bigr]
=s'(p)\,p(1-p) + s(p)\,(1-2p).
\end{equation}

\textbf{Step 2: Analyze the concave case of $f''\le 0$.}

If $f''\le 0$, then $s'(p)=-f''(p)\ge 0$. Consider two regions:

\emph{Region $p\in(0,\tfrac{1}{2})$:}
Here, $1-2p>0$. Both terms in $W_f'(p)$ are nonnegative:
\begin{equation}
\label{eq:app:59}
s'(p)\,p(1-p)\ge 0,\qquad s(p)\,(1-2p)>0.
\end{equation}
Since $s(p)>0$, we have $W_f'(p)>0$ on $(0,\tfrac{1}{2})$. Thus, $W_f$ is strictly increasing on this interval.

\emph{Conclusion:}
Any interior maximizer must lie in $[\tfrac{1}{2},1]$ (since $W_f$ is increasing on $(0,\tfrac{1}{2})$).

\textbf{Step 3: Analyze the convex case of $f''\ge 0$.}

If $f''\ge 0$, then $s'(p)=-f''(p)\le 0$. Consider the region $p\in(\tfrac{1}{2},1)$:
Here, $1-2p<0$. Both terms in $W_f'(p)$ are nonpositive:
\begin{equation}
\label{eq:app:60}
s'(p)\,p(1-p)\le 0,\qquad s(p)\,(1-2p)<0.
\end{equation}
Thus, $W_f'(p)<0$ on $(\tfrac{1}{2},1)$, meaning $W_f$ is strictly decreasing on this interval.

\emph{Conclusion:}
Any interior maximizer must lie in $[0,\tfrac{1}{2}]$.
\end{proof}


\subsection{Quantifying Loss Reduction via Gradient Flow}

In this subsection, we characterize when one objective achieves greater loss reduction than another in the gradient flow dynamics.

\begin{definition}[Gradient Flow Rate]
\label{def:gradient_flow}
Let $\cR(\mtheta)$ denote the population risk.
For an objective $f$, define the gradient flow rate at ${\mtheta_0}$ by
\begin{equation}
\label{eq:app:61}
\dot{\cR}(\mtheta_t^{(f)})\big|_{t=0}
\;\coloneqq\;
\lim_{\eta\searrow 0}\frac{\cR({\mtheta_0})-\cR({\mtheta_0}-\eta\nabla\cL_f({\mtheta_0}))}{\eta}
\;=\;
\langle \nabla\cR({\mtheta_0}),\,\nabla\cL_f({\mtheta_0})\rangle.
\end{equation}
A larger $\dot{\cR}$ indicates faster initial improvement in the population risk.
\end{definition}

\begin{assumption}[Simplified Feature Geometry]
\label{assump:simplified_feature}
To isolate the effect of objective choice, we assume:
\begin{enumerate}
\item The feature Jacobian $\Phi(c,y)\triangleq\nabla_{\mtheta} z_{\mtheta_0}(c,y)$ has the property that $\Phi(c)^\top\Phi(c)\approx I$ (approximately orthonormal).
\item Both the training distribution $T$ and the true distribution $r$ are one-hot (i.e., deterministic targets).
\end{enumerate}
These assumptions remove irrelevant conditioning factors and allow us to focus on the essential differences between objectives.
\end{assumption}

\begin{theorem}[Sufficient Condition for Objective Ordering]
\label{thm:objective_ordering}
Consider two objectives $f_1$ and $f_2$, both differentiable and nonincreasing.
Define $d(p)\triangleq f_2'(p)-f_1'(p)$.
Suppose $d(p)\le 0$ for all $p\in(0,1)$.
Under ~\assumpref{assump:simplified_feature}, the gradient flow rates satisfy:
\begin{equation}
\label{eq:app:62}
\dot{\cR}(\mtheta_t^{(2)})\big|_{t=0} - \dot{\cR}(\mtheta_t^{(1)})\big|_{t=0}
\;=\;
\E_{(c,\tilde{y})\sim T}\!\left[
q_{y^*}\,q_{\tilde{y}}\,\bigl(f_2'(q_{\tilde{y}})-f_1'(q_{\tilde{y}})\bigr)\,\langle e_{y^*}-q,\,e_{\tilde{y}}-q\rangle
\right],
\end{equation}
where $y^*$ is the true target, $\tilde{y}$ is the supervised label, $q=p_{{\mtheta_0}}(\cdot\mid c)$ is the base model distribution, and $e_y$ is the one-hot vector for token $y$.

\emph{In the model-strong end:} If $q_{y^*}$ and $q_{\tilde{y}}$ are both large, and $y^*=\tilde{y}$ with high probability, then
\begin{equation}
\label{eq:app:63}
\langle e_{y^*}-q,\,e_{\tilde{y}}-q\rangle=\|e_{y^*}-q\|^2>0.
\end{equation}
Since $d(q_{\tilde{y}})\le 0$, we have $\dot{\cR}(\mtheta_t^{(2)})\big|_{t=0}\ge \dot{\cR}(\mtheta_t^{(1)})\big|_{t=0}$.

\emph{In the model-weak end:} If $q$ is nearly uniform, then for $y^*\ne\tilde{y}$,
\begin{equation}
\label{eq:app:64}
\langle e_{y^*}-q,\,e_{\tilde{y}}-q\rangle\approx -\frac{1}{|\cV|}<0.
\end{equation}
Thus, the sign flips, and $\dot{\cR}(\mtheta_t^{(2)})\big|_{t=0}\le \dot{\cR}(\mtheta_t^{(1)})\big|_{t=0}$.
\end{theorem}

\begin{proof}
\textbf{Step 0: Write the gradient flow difference.}

By Definition~\ref{def:gradient_flow},
\begin{equation}
\label{eq:app:65}
\dot{\cR}(\mtheta_t^{(i)})\big|_{t=0}
=\langle \nabla\cR({\mtheta_0}),\,\nabla\cL_{f_i}({\mtheta_0})\rangle.
\end{equation}
Thus,
\begin{equation}
\label{eq:app:66}
\dot{\cR}(\mtheta_t^{(2)})\big|_{t=0} - \dot{\cR}(\mtheta_t^{(1)})\big|_{t=0}
=\langle \nabla\cR({\mtheta_0}),\,\nabla\cL_{f_2}({\mtheta_0})-\nabla\cL_{f_1}({\mtheta_0})\rangle.
\end{equation}

\textbf{Step 1: Use the one-hot assumption and feature orthonormality.}

Under ~\assumpref{assump:simplified_feature}, with $\Phi^\top\Phi\approx I$ and one-hot $T$ and $r$, the gradients simplify to:
\begin{equation}
\label{eq:app:67}
\nabla\cR({\mtheta_0})
=\E_c\!\left[\Phi(c)\,v_*(c)\right],
\qquad
v_*(c)=q_{y^*}(e_{y^*}-q).
\end{equation}
\begin{equation}
\label{eq:app:68}
\nabla\cL_{f_i}({\mtheta_0})
=\E_c\!\left[\Phi(c)\,v_i(c)\right],
\qquad
v_i(c)=q_{\tilde{y}}f_i'(q_{\tilde{y}})(e_{\tilde{y}}-q).
\end{equation}

\textbf{Step 2: Compute the inner product.}

Since $\Phi^\top\Phi\approx I$, we have
\begin{equation}
\label{eq:app:69}
\langle \nabla\cR,\,\nabla\cL_{f_i}\rangle
\approx\E_c\!\left[\langle v_*,\,v_i\rangle\right]
=\E_c\!\left[q_{y^*}\,q_{\tilde{y}}f_i'(q_{\tilde{y}})\,\langle e_{y^*}-q,\,e_{\tilde{y}}-q\rangle\right].
\end{equation}

\textbf{Step 3: Take the difference.}
\begin{equation}
\label{eq:app:70}
\dot{\cR}^{(2)}-\dot{\cR}^{(1)}
=\E_c\!\left[
q_{y^*}\,q_{\tilde{y}}\,\bigl(f_2'(q_{\tilde{y}})-f_1'(q_{\tilde{y}})\bigr)\,\langle e_{y^*}-q,\,e_{\tilde{y}}-q\rangle
\right].
\end{equation}

\textbf{Step 4: Analyze the sign in model-strong and model-weak ends.}

\emph{Model-strong:} $q_{y^*}\approx 1$ and $y^*=\tilde{y}$ w.h.p., so
\begin{equation}
\label{eq:app:71}
\langle e_{y^*}-q,\,e_{\tilde{y}}-q\rangle=\|e_{y^*}-q\|^2>0.
\end{equation}
Since $d(q_{\tilde{y}})\le 0$, the expectand is $\ge 0$, hence $\dot{\cR}^{(2)}\ge\dot{\cR}^{(1)}$.

\emph{Model-weak:} $q\approx\mathbf{1}/|\cV|$ (uniform), so for $y^*\ne\tilde{y}$,
\begin{equation}
\label{eq:app:72}
\langle e_{y^*}-q,\,e_{\tilde{y}}-q\rangle
=\bigl(\mathbf{1}_{y^*}-\tfrac{1}{|\cV|}\mathbf{1}\bigr)^\top\bigl(\mathbf{1}_{\tilde{y}}-\tfrac{1}{|\cV|}\mathbf{1}\bigr)
=-\frac{1}{|\cV|}<0.
\end{equation}
Thus, $\dot{\cR}^{(2)}\le\dot{\cR}^{(1)}$.
\end{proof}

\subsection{Experimental Setup}
\label{exp_appendix}
The details of our experimental setup are as follows. All SFT experiments were conducted using the \textbf{VERL toolkit}~\cite{sheng2024hybridflow}. The optimizer was fixed to \texttt{AdamW}, with a base learning rate of $5 \times 10^{-5}$ for all models except \textbf{LLaMA-3.8B}, which used a base learning rate of $2 \times 10^{-5}$. We employed \emph{cosine decay learning rate scheduling} with a warm-up ratio of 0.1, and all training runs were performed for 1 epoch. All tasks were executed on a single node equipped with 8 H20 GPUs.

\newpage

\begin{table*}[t]

\caption{Performance of multi-epoch training on MATH using LLaMA-3.1-8B (\textbf{bold}: best). Results show that Cayley-Trans and DEFT consistently outperform both $-\log p$ and $-p$ across all epochs.
}
\centering
\scalebox{0.85}{
\begin{tabular}{cccccccc}
\toprule
\textbf{Method}                        & \textbf{Epoch} & \textbf{Minerva Math} & \textbf{Olympiad Bench} & \textbf{AMC23}   & \textbf{Math500}  & \textbf{AIME24}  & \textbf{Avg}            \\
\midrule
\multirow{4}{*}{-logp}        & 1     & 6.15    & 3.11            & 5.63  & 17.58     & 0.00   & 6.49           \\
                              & 2     & 9.20    & 5.30            & 2.50  & 20.40     & 0.00   & 7.48           \\
                              & 3     & 8.10    & 6.70            & 10.00 & 27.20     & 0.00   & 10.40          \\
                              & 4     & 7.70    & 8.30            & 5.00  & 26.00     & 0.00   & 9.40           \\
                              \midrule
\multirow{4}{*}{-p}           & 1     & 9.90    & 5.69            & 10.00 & 23.91     & 0.62   & 10.02          \\
                              & 2     & 8.21    & 5.97            & 11.40 & 24.94     & 0.00   & 10.10          \\
                              & 3     & 8.72    & 7.06            & 12.34 & 25.50     & 0.41   & 10.81          \\
                              & 4     & 8.74    & 7.42            & 11.72 & 25.90     & 0.00   & 10.76          \\
                              \midrule
\multirow{4}{*}{EAFT}         & 1     & 6.20    & 3.44            & 7.34  & 16.84     & 0.00   & 6.76           \\
                              & 2     & 8.78    & 4.68            & 7.19  & 23.25     & 0.00   & 8.78           \\
                              & 3     & 10.46   & 6.02            & 8.59  & 27.23     & 0.41   & 10.54          \\
                              & 4     & 10.59   & 6.54            & 8.75  & 28.26     & 0.41   & 10.91          \\
                              \midrule
                              
\multirow{4}{*}{Cayley-Trans} & 1     & 10.46   & 6.18            & 9.06  & 25.99     & 0.21   & 10.38          \\
                              & 2     & 12.27   & 8.38            & 12.50 & 31.49     & 0.83   & 13.09          \\
                              & 3     & 13.76   & 9.34            & 11.88 & 34.73     & 0.41   & \textbf{14.02} \\
                              & 4     & 13.70   & 9.13            & 11.72 & 34.74     & 0.62   & 13.98          \\
                              \midrule
                         
\multirow{4}{*}{DEFT}        & 1     & 10.21   & 6.43            & 12.81 & 26.26     & 1.03   & 11.35          \\
                              & 2     & 10.40   & 7.46            & 9.38  & 28.86     & 0.62   & 11.34          \\
                              & 3     & 9.81    & 7.87            & 10.00 & 30.61     & 1.03   & 11.86          \\
                              & 4     & 10.23   & 7.79            & 9.22  & 30.83     & 0.62   & 11.74    \\  
                              \midrule
\end{tabular}
}
\label{multi_epoch}
\end{table*}

\begin{table*}[t]
\centering
\caption{Performance on General (medical) benchmarks under the Mixed Setting The best results are highlighted in bold. The best results are highlighted in \textbf{bold}. The second-best results are \underline{underlined}.}
\scalebox{0.75}{
\begin{tabular}{lccccccccccc}
\toprule
\multicolumn{1}{c}{Method} & \multicolumn{1}{c}{HLE} & \multicolumn{1}{c}{MMLU-P} & \multicolumn{1}{c}{GPQA} & \multicolumn{1}{c}{MedMC} & \multicolumn{1}{c}{MedQA} & \multicolumn{1}{c}{PubMed} & \multicolumn{1}{c}{MedX} & \multicolumn{1}{c}{Lancet} & \multicolumn{1}{c}{MedB(4)} & MedB(5) & Avg   \\
\midrule
-logp                      & 15.82                   & 33.94                      & 42.31                    & 40.09                       & 40.69            & 65.00                      & 9.45                           & 38.83                      & 33.44                       & 25.97   & \underline{34.55} \\
-p                         & 17.09                   & 31.79                      & 24.87                    & 42.43                       & 39.12                     & 58.60                      & 10.28                          & 41.02                      & 35.06                       & 27.92   & 32.82 \\
EAFT                       & 15.82                   & 35.57                      & 26.41                    & 41.45                       & 41.79            & 65.40                      & 9.66                           & 43.45                      & 33.77                       & 25.32   & 33.86 \\
\rowcolor{cyan!10}
Cayley-Trans               & 16.46                   & 35.70                      & 34.36                    & 40.50                       & 41.24                     & 64.20                      & 9.94                           & 38.35                      & 35.39                       & 24.35   & 34.05 \\
\rowcolor{cyan!10}
DEFT                       & 18.35                   & 35.77                      & 35.38                    & 40.14                       & 39.83                     & 64.90                      & 12.49                          & 45.39                      & 36.36                       & 30.52   & \textbf{35.91} \\
\bottomrule
\end{tabular}
}
\label{mix_general}
\end{table*}

\begin{table*}[t]
\caption{Performance on MATH benchmarks under the Mixed Setting. The best results are highlighted in \textbf{bold}. The second-best results are \underline{underlined}.}
\centering
\scalebox{0.9}{
\begin{tabular}{ccccccc}
\toprule
Methods      & Math500 & Minerva Math & Olympiad Bench & AIME24 & AMC23 & Avg   \\
\midrule
-logp        & 24.92   & 8.06         & 5.89           & 0.41   & 9.38  & 9.73  \\
-p           & 29.61   & 10.33        & 7.79           & 0.21   & 17.53 & \textbf{13.09} \\
EAFT         & 26.00   & 6.20         & 7.00           & 0.00   & 12.50 & 10.34 \\
\rowcolor{cyan!10}
Cayley-Trans & 29.80   & 11.00        & 6.80           & 0.00   & 7.50  & 11.02 \\
\rowcolor{cyan!10}
DEFT         & 29.10   & 9.79         & 7.71           & 0.41   & 13.44 & \underline{12.09} \\
\bottomrule
\end{tabular}
}
\label{mix_math}
\end{table*}



\begin{figure*}[t]
    \centering
    \includegraphics[width=\textwidth]{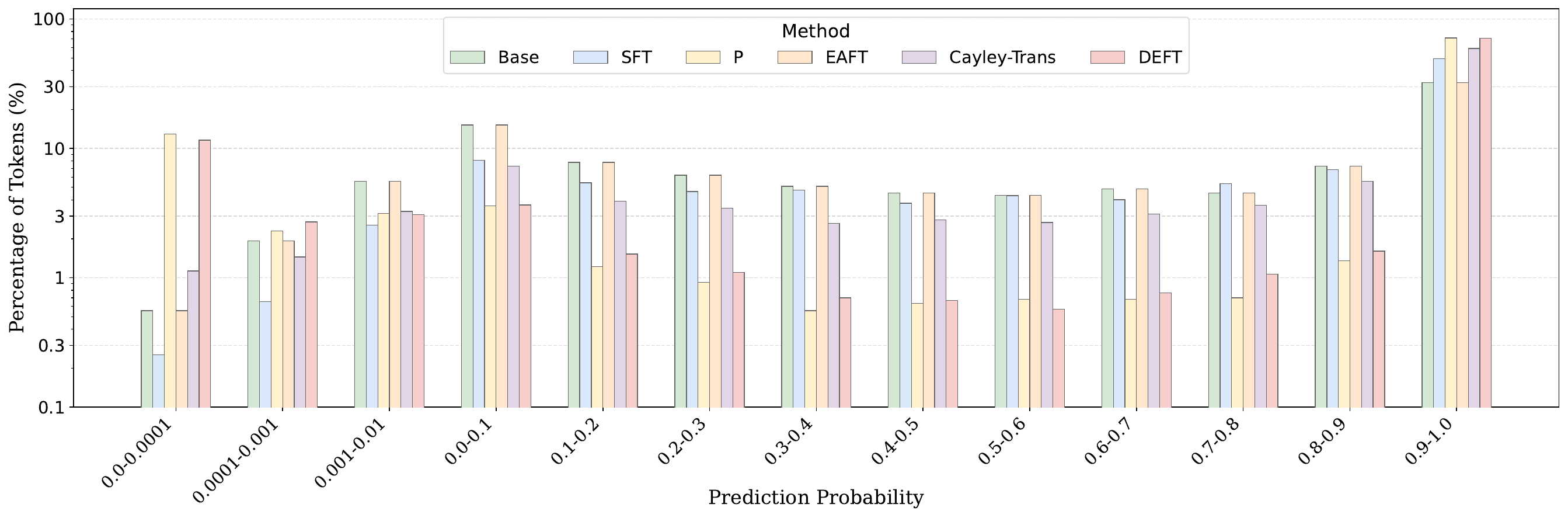}
    \caption{Token probability distributions on the training set under different objectives in the Model-Intermediate regime.}
    \label{fig:token_prob_dist_medical}
\end{figure*}

\end{CJK*}
\end{document}